\documentclass{article}

    \PassOptionsToPackage{numbers, compress}{natbib}

\usepackage{microtype}
\usepackage{wrapfig}
\usepackage{graphicx}
\usepackage{subfigure}
\usepackage{booktabs} 
\usepackage{amsfonts,eucal,amsbsy,amsopn,amsmath, amsthm}
\usepackage{mathtools}
\usepackage{tikz}
\usepackage{bm}

\newcommand{\R}{\ensuremath{\mathbb{R}}}

\DeclareMathOperator{\E}{\mathbb{E}}

\newcommand{\op}{\ensuremath{\mathsf{OP}}\xspace}

\theoremstyle{theorem}
\newtheorem{theorem}{Theorem}

\theoremstyle{definition}

\renewenvironment{proof}{{\bf Proof:}}{\hfill\rule{2mm}{2mm}}

\newcommand{\wh}{\widehat}
\renewcommand{\op}{\operatorname}

\usepackage{hyperref}
\usepackage[capitalise]{cleveref}

\usepackage[final]{neurips_style/neurips_2022}

\title{Deep Multi-Modal Structural Equations For Causal Effect Estimation With Unstructured Proxies}

\author{
  Shachi Deshpande$^{1, 2}$,  Kaiwen Wang$^{1, 2}$, Dhruv Sreenivas$^{2}$, Zheng Li$^{1, 2}$, Volodymyr Kuleshov$^{1, 2}$ \\
   Department of Computer Science, Cornell Tech$^{1}$ and Cornell University$^{2}$\\
  \texttt{\string{ssd86, kw437, ds844, zl634, kuleshov\string}@cornell.edu}\\
  }
  
\date{%
    $^1$Organization 1\\%
    $^2$Organization 2\\[2ex]%
    \today
}

\begin{document}

\maketitle

\begin{abstract}
Estimating the effect of an intervention 
from observational data
while accounting for confounding variables
is a key task in causal inference. 
%
%
Oftentimes, 
the confounders are unobserved, but we have access to large amounts of 
additional
unstructured data (images, text) that contain valuable proxy signal about the missing confounders.
This paper argues that leveraging this unstructured data 
can greatly improve
the accuracy of causal effect estimation. 
Specifically, we introduce deep multi-modal structural equations, a generative model
for causal effect estimation
in which 
confounders are latent variables
and unstructured data are proxy variables.
This model supports multiple multi-modal proxies (images, text) as well as missing data.
We empirically demonstrate that our approach outperforms existing methods based on propensity scores and corrects for  confounding using unstructured inputs on tasks in genomics and healthcare.
Our methods can potentially support the use of large amounts of  data that were previously not used in causal inference.
\end{abstract}

\section{Introduction}

An important goal of causal inference is to understand from observational data the causal effect of performing an intervention---e.g., the effect of a behavioral choice on an individual's health \citep{pearl2009causality}. As an initial motivating example for this work, consider the problem of determining the effect of smoking on an individual's risk of heart disease.

This problem is complicated by the presence of 
confounders: 
e.g., it is possible that individuals who smoke have a higher likelihood to be sedentary, which is a lifestyle choice that also negatively impacts their heart disease risk. 
If the individual's lifestyle is available to us as a well-defined feature, we may adjust for this factor while computing treatment effects 
\citep{pearl2009causality}.
However, confounders are often not observed and not available in the form of features, making accurate causal inference challenging.

Oftentimes, datasets in domains such as medicine or genomics come with large amounts of {\em unstructured data}---e.g., medical images, clinical notes, wearable sensor measurements \citep{bycroft2018uk}. This data contains strong proxy signal about unobserved confounding factors---e.g., wearable sensor measurements can help reveal individuals who are sedentary. However, existing causal inference methods are often not able to leverage this ``dark data" for causal effect estimation \citep{kallus2020deepmatch}.

The goal of this paper is to develop methods in causal inference that improve the estimation of causal effects in the presence of unobserved confounders by leveraging additional sources of unstructured multi-modal data, such as images and text. 
For example, given time series from patients' wearables, our methods may disentangle the effects of being sedentary from the effects of smoking by
using clusters in the sensor measurements (which would correspond to groups of active and sedentary individuals) as a proxy for the patients' lifestyles and
without requiring explicit lifestyle features.



Concretely, our paper formalizes the task of estimating causal effects using rich, unstructured, multi-modal proxy variables (e.g. images, text, time series)
and introduces deep multi-modal structural equations, a generative model in which 
confounders are latent variables.
This model 
can perform causal effect estimation with missing data
by leveraging approximate variational inference and learning algorithms \citep{blei2017variational} developed for this task.
Previous methods relied on propensity scoring with neural approximators \citep{pryzant2017predicting,veitch2019using,veitch2020adapting}, which may output volatile probabilistic outputs that lead to unreliable effect estimates \citep{kallus2020deepmatch}.
Our methods are generative, and thus naturally sidestep some of the aforementioned instabilities, support high-dimensional treatment variables and incomplete data, and as a result are applicable to a broader class of causal inference problems. 

We evaluate our methods on an important real-world causal
inference task---estimating the effects of genetic mutations in genome-wide association studies (GWASs)---as well as on benchmarks derived from popular causal inference datasets.
The intervention variables in a GWAS are high-dimensional genomic sequences, hence existing methods based on propensity scoring are not easily applicable.
In contrast, we demonstrate that our algorithms naturally leverage high-dimensional genetic and environmental data (e.g., historical weather time series) and can discover causal genetic factors in plants and humans more accurately than existing GWAS analysis methods.
%



\paragraph{Contributions}
In summary, this paper makes three contributions: (1) we define the task of estimating causal effects using rich, unstructured multi-modal proxy variables; (2) we introduce deep multi-modal structural equations, a generative model tailored to this problem, and we describe associated variational learning and inference algorithms; (3) we demonstrate 
on an important real-world  problem (GWAS)
that unstructured data can improve causal effect estimation, enabling the use of large amounts of ``dark data" that were previously not used in causal inference.

\section{Background}




\paragraph{Notation}

Formally, we are given an observational dataset $\mathcal{D}=\{(x^{(i)},y^{(i)},t^{(i)})\}_{i=1}^n$ consisting of $n$ individuals, each characterized by features $x^{(i)} \in \mathcal{X} \subseteq \mathbb{R}^d$, a binary treatment $t^{(i)} \in \{0,1\}$, and a scalar outcome $y^{(i)} \in \mathbb{R}$. We initially assume binary treatments and scalar outcomes, and later discuss how our approach naturally extends beyond this setting.
We also use $z^{(i)} \in \mathbb{R}^p$ to model latent confounding factors that influence both the treatment and the outcome \citep{louizos2017causal}. 
We are interested in recovering the true effect of $T=t$ in terms of its conditional average treatment effect (CATE), also known as the individual treatment effect (ITE) and average treatment effect (ATE).
\begin{align}
Y[x,t] =  \mathbb{E}[Y | X =x, \text{do}(T=t)]
&& \text{ITE}(x) = Y[x,1] - Y[x,0]  
&& \text{ATE} = \mathbb{E}[\text{ITE}(X)],
\end{align}
where $ \text{do}(\cdot)$ denotes an intervention \citep{pearl2000models}.
Many methods for this task rely on propensity scoring \citep{pryzant2017predicting,veitch2019using,veitch2020adapting}, which uses a model $p(t|x)$ to assign weights to individual datapoints; however, when $x$ is high-dimensional and unstructured, a neural approximator for $p(t|x)$ may output volatile and miscalibrated probabilities close to $\{0,1\}$ that lead to unreliable effect estimates \citep{kallus2020deepmatch}.

\paragraph{Structural Equations}

An alternative approach 
are structural equation models of the form 
\begin{align}
x = f_1(z, \varepsilon_1) && t = f_2(z, \varepsilon_2) && y = f_3(z,t, \varepsilon_3),
\label{eqn:structural_eqn}
\end{align}
where $Z \sim p(Z)$ is drawn from a prior and the $\varepsilon_i$ are noise variables drawn independently from their distributions \citep{duncan2014introduction}.
Structural equations define a {\em generative model} $p(x,y,z,t)$ of the data.
When this model encodes the true dependency structure of the data distribution, we can estimate the true effect of an intervention by clamping $t$ to its desired value and drawing samples. 

\paragraph{Deep Structural Equations}

Equations \ref{eqn:structural_eqn} can be parameterized with deep neural networks, which yields deep structural equation models \citep{tran2017implicit,louizos2017causal}. Expressive neural networks may learn a more accurate model of the true data distribution on large datasets, which improves causal effect estimation.
Such models have been used for GWAS analysis \citep{tran2017implicit}
 and to correct for proxy variables \citep{louizos2017causal}.


\section{Causal Effect Estimation With Unstructured Proxy Variables}

Oftentimes, datasets in domains such as medicine or genomics come with large amounts of unstructured data (medical images, clinical notes), which contains strong proxy signal about unobserved confounding factors.
Our paper seeks to develop methods that leverage unstructured data within causal inference. We start by formalizing this task as causal effect estimation with unstructured proxy variables; these proxies may come from multiple diverse modalities (images, text). 

\subsection{Task Definition}
\label{sec:definition}

Formally, consider a causal inference dataset $\mathcal{D}=\{(x^{(i)},y^{(i)},t^{(i)})\}_{i=1}^n$ in which $x^{(i)} = (x_1^{(i)}, x_2^{(i)}, \ldots, x_m^{(i)})$ is a vector of $m$ distinct input modalities $x_j^{(i)} \in \mathcal{X}_j$ (e.g., images, text, time series, etc.). In other words, $\mathcal{X} = \mathcal{X}_1 \times \ldots \times \mathcal{X}_m$, where each $\mathcal{X}_j$ corresponds to a space of images, time series, or other unstructured modalities. Here, $t^{(i)} \in \mathcal{T}$ (binary or continuous) is the treatment and $y^{(i)} \in \mathcal{Y}$ is the output.
Some modalities may also be missing at training or inference time.

We are interested in recovering the true effect of $t$ in terms of the individual and average treatment effects. We are specifically interested in estimating the individual treatment effect (ITE) from arbitrary subsets of modalities $\mathcal{M} \subseteq \{1,2,...,m\}$, indicating that  certain inputs  may be missing at test time.
\begin{align}
Y[x,t,\mathcal{M}] =  \mathbb{E}[Y|\text{do}(T=t), X_j=x_j \text{ for $j$ in } \mathcal{M}] &&
\text{ITE}(x,\mathcal{M}) =  Y[x,t=1,\mathcal{M}] - Y[x,t=0,\mathcal{M}]
\label{eqn:ite_calc}
\end{align}
%
To help make this setup more concrete, we define two motivating applications.

\paragraph{Healthcare}
Consider the task of determining the effect of smoking on heart disease from an observational dataset of patients.
The observational study may contain additional unstructured data about individuals, e.g., clinician notes, medical images, wearable sensor data, etc. 
This data may hold information about hidden confounders: for example, raw wearable sensor data can be clustered to uncover sedentary and active indivudals, revealing a latent confounding factor, sedentary lifestyle. 

\paragraph{Genomics}
Consider the problem of estimating the effects of genetic variants via a genome-wide association study (GWAS). Modern GWAS datasets in plants or humans feature large amounts of unstructured inputs \citep{bycroft2018uk}: clinical notes, medical records, meteorological time series. 
For example, historical weather data (e.g., precipitation, wind strength, etc.) can reveal distinct climatic regions that affect plant phenotypes and whose confounding effects should be corrected for in a GWAS \citep{weigel20091001}. 

\section{Deep Structural Equations for Causal Effect Estimation}


Next, we derive models and inference algorithms for the task of causal effect estimation with unstructured proxy variables.
Our approach uses deep structural equations to extract confounding signal from the multi-modal proxies $x_j^{(i)}$. 
We use neural networks because they naturally handle unstructured modalities via specialized architectures (e.g., convolutions for images) that can learn high-level representations over raw unstructured inputs (e.g., pixels).







Parameterizing structural equations with neural networks also presents challenges: they induce complex latent variable models that require the development of efficient approximate inference algorithms \citep{duncan2014introduction,blei2017variational}.
We present an instantiation of Equations \ref{eqn:structural_eqn} that admits such efficient algorithms.

\subsection{Deep Multi-Modal Structural Equations}
We start by introducing deep
multi-modal structural equations (DMSEs), a generative model 
for estimating causal effects 
in which confounders are
latent variables and unstructured data are proxy variables. We define a DMSE model as follows:
%
\begin{align}
\label{eqn:dmse}
z \sim \mathcal{N}(0_p,I_p) &&
x_j \sim p_{x_j}(\,\cdot\, ; \theta_{x_j}(z))  \;\, \forall j &&
t \sim \mathrm{Ber}(\pi_t(z)) && 
y \sim p_y(\, \cdot \, ; \theta_y(z,t)),
\end{align}
where $p_{x_j}, p_y$ are probability distributions with a tractable density over $x_j$ and $y$, respectively, and the $\theta_{x_j}, \theta_{y}$ are the parameters of $p_{x_j}, p_y$. The $\theta_{x_j}, \theta_{y}$ are themselves functions of $z, t$ parameterized by neural networks---e.g.,
when $x_j$ is Gaussian, the $\theta_{x_j}(z)$ are a mean and a covariance matrix $\mu_{x_j}(z), \Sigma_{x_j}(z)$ that are parameterized by a neural network as a function of $z$ (see \citep{kingma2013autoencoding,louizos2017causal}). 
Note that other modeling choices (e.g., Bernoulli distributions for discrete variables) are also possible.

Note that the models for $\theta_{x_j}(z)$ can benefit from domain-specific neural architectures---e.g., a convolutional parameterization for $\mu_{x_j}(z), \sigma_{x_j}(z)$ as a function of $z$ is more appropriate when the $x_j$ are images. See Appendix \ref{apdx-architecture} for details on recommended architectures. 

While Section \ref{sec:definition} defines $y, t$ as scalars following existing literature \citep{louizos2017causal,wang2019blessings,veitch2019using}, DMSEs can also define a model with high-dimensional $y,t$---we simply choose the distributions over $y,t$ to be multi-variate.
Our inference and learning algorithms will remain unchanged, except for the parameterization of specific approximate posteriors (e.g., $q(z|y,t)$). In fact, we apply DMSEs to high-dimensional $t$ in our GWAS experiments in Section \ref{sec:experiments}.
Note that this is a setting where existing propensity scoring methods (which learn a model of $p(t|x)$) are not directly applicable \citep{pryzant2017predicting,veitch2019using}---our approach, on the other hand, can easily be used with high-dimensional $t$ on tasks like GWAS analysis.

\paragraph{Dependency Structure}
Equations (\ref{eqn:dmse}) define a density 
$p(z) p(t|z)p(y|z,t) \prod_{j=1}^m p(x_j|z)$. 
Note that 
each proxy $x_j$ is independent of the others conditioned on $z$. 
Figure~\ref{fig:multimodal_structure} shows these dependencies as solid lines. 
In our setting, the $x_j$ represent image pixels, waveform measurements, etc; thus, they should not directly influence each other or $y, t$.
For example, it would not make sense for the pixels of an image to causally influence the samples of a waveform measurement---they influence each other only through a latent confounder $z$ (e.g., patient health status), and the $x_j$ are therefore assumed to be conditionally independent given $z$.
These independence assumptions will also enable us to derive efficient stochastic variational inference algorithms, as we show below.

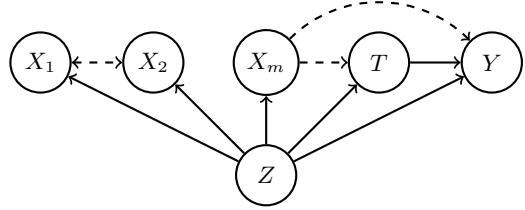
\begin{wrapfigure}{r}{7cm}
    \vspace{-8mm}
    \centering
    \begin{small}
    \begin{tikzpicture}[thick, node distance={15mm}, minimum size=0.8cm, main/.style = {draw, circle}] 
    \node[main] (1) {$X_1$};
    \node[main] (2) [right of=1] {$X_2$};
    \node[main] (3) [right of=2] {$X_m$};
    \node[main] (4) [right of=3] {$T$};
    \node[main] (6) [right of=4] {$Y$};
    \node[main] (5) [below of=3]{$Z$}; 
    \draw [->] (5) -- (1);
    \draw [->] (5) -- (2);
    \draw [->] (5) -- (3);
    \draw [->] (5) -- (4);
    \draw [->] (4) -- (6);
    \draw [->] (5) -- (6);
    \draw [dashed, <->] (1) -- (2);
    \draw [dashed, ->] (3) to [out=45,in=135] (6);
    \draw [dashed, ->] (3) -- (4);
    \end{tikzpicture} 
    \end{small}
     \caption{Causal graph for the DMSE model. Solid lines depict dependencies between variables. Appendix \ref{app:alt_graphs} contains simple extensions of DMSEs that support extra dependencies (dashed lines).}
    \label{fig:multimodal_structure}
    \vspace{-5mm}
\end{wrapfigure}

In certain settings, we may also want to model observed confounders as well as additional {\em structured} proxies $x_j$ that have direct causal effects on each other and on $y, t$. Figure \ref{fig:multimodal_structure} shows these additional dependencies as dashed lines.
Our model admits simple generalizations that support these modeling assumptions.
In brief, our variational inference algorithms can be conditioned on the observed variables, and mutually dependent sets of proxies can be treated as a group represented by one high-dimensional variable. 
See Appendix \ref{app:alt_graphs} for details.

\subsection{Approximate Inference and Learning Algorithms}
\label{subsection:approx_inference}


In the full multi-modal setting, the $x_j$ are conditionally independent given $z$ (Figure~\ref{fig:multimodal_structure}), which enables us to apply efficient algorithms inspired by \citet{wu2018multimodal}. These algorithms offer the following improvements: (1) we may perform learning, inference, and causal estimation with missing modalities $x_j$; (2) the process for performing causal inference  does not require training auxiliary inference networks as in previous work \citep{louizos2017causal}.


The DMSE model induces a tractable joint density $p(x,y,t,z)$, which allows us to fit its parameters
using stochastic variational inference by optimizing the evidence lower bound (ELBO):
\begin{align} 
\text{ELBO}_X & = \sum_{i=1}^n \mathbb{E}_{q} \bigl[\sum_{j=1}^m\log p(x_j^{(i)}|z) + \log p(y^{(i)},t^{(i)},z) - \log q(z|x^{(i)},y^{(i)},t^{(i)}) \bigr],
\end{align}
where $p(y^{(i)},t^{(i)},z) = p(y^{(i)} | t^{(i)},z) p(t^{(i)}|z) p(z)$
and $q(z|x^{(i)}, y^{(i)}, t^{(i)})$ is the approximate variational posterior. We assume a total of $m$ modalities.

\paragraph{Structured Multi-Modal Variational Inference}

We may use the independence structure of $p$ (Figure~\ref{fig:multimodal_structure}) to derive an efficient structured form for $q$. 
First, observe that because the true posterior factorizes as $p(z|x, t, y) \propto (p(z| t, y)\prod_{j=1}^{m} p(z|x_j))/\prod_{j=1}^{m-1} p(z)$,
the optimal approximate posterior $q$ must also factorize as
$q(z|x, t, y) \propto (q(z| t, y)\prod_{j=1}^{m} q(z|x_j))/\prod_{j=1}^{m-1} p(z)$. 
This decomposition implies that we can maintain the optimal structure of $q$ by training modality-specific inference networks $\tilde q(z|t, y)$ and $\tilde q(z|x_j)$ such that $q(z|x_j) = \tilde q(z|x_j) p(z)$ and $q(z|t, y) = \tilde q(z|t, y) p(z)$ and by defining a joint posterior as
\begin{align}
    q(z|x, y, t) \propto p(z) \tilde q(z| y, t)\prod_{j=1}^{m} \tilde q(z|x_j).
    \label{eqn:product_of_experts}
\end{align}
This network can be seen as a product of experts (PoE)~\citep{wu2018multimodal}.

Computing the density $q$ is in general not possible. However, because $p(z)$, $q(z|t, y)$ and $q(z|x_j)$ are Gaussians, we may use the fact that a product of Gaussians with means $\mu_i$ and covariances $V_i$ is $\mu = (\sum \mu_i T_i ) / (\sum T_i)$ and  $V = (\sum_i T_i)^{-1}$, where $T_i = 1/V_i$. Thus, computing $q(z|x, t, y)$ for any subset of modalities is  possible without having to train an inference network for each subset of modalities separately. 

\paragraph{A Multi-Modal ELBO}
Training this architecture with $\text{ELBO}_X$ will not necessarily yield good single-modality inference networks $\tilde q(z|x_j)$, while training with each modality separately will prevent the network from learning how the modalities are related to each other. Hence, we train our model with a sub-sampled ELBO objective \citep{wu2018multimodal} that is computed on the full set of modalities, each individual modality, and a few subsets of modalities together. For this, we randomly pick at each gradient step $s$ non-empty subsets $\{\mathcal{M}_k\}_{k=1}^{s}$ of the set of modalities $\mathcal{M}_k\subseteq \{1,2,...,m\}$. The final objective is
$
    \text{ELBO}_X + \sum_{j=1}^m \text{ELBO}_{\{j\}} + \sum_{k=1}^{s} \text{ELBO}_{\mathcal{M}_k},
$
where
\begin{align*}
\text{ELBO}_{\mathcal{M}}  =\sum_{i=1}^n \mathbb{E}_{q} & \left[ \sum_{j \in \mathcal{M}}\log p(x_j^{(i)},z) \right. + \log p(y^{(i)},t^{(i)},z) \left. - \log ( p(z) \tilde q(z| y^{(i)}, t^{(i)})\prod_{j\in \mathcal{M}} \tilde q(z|x_j^{(i)}) ) \right]
\end{align*}
Our model doesn't train extra auxiliary inference networks, unlike that of \citet{louizos2017causal}.

\subsection{Deep Gaussian Structural Equations}

The DMSE model can be simplified in settings 
in which there is only one type of proxy $x$ (i.e., $m=1$). 
This simplified model, which we call deep Gaussian structural equations (DGSEs), has a tractable joint density
$p(x,y,t,z) = p(z)p(x|z)p(t|z)p(y|z,t),$ 
where the latent $z$ is Gaussian.

The DGSE model can also be fit using stochastic variational inference by optimizing the ELBO objective
$
\sum_{i=1}^n \mathbb{E}_{q} \left[ \log p(x^{(i)},y^{(i)},t^{(i)},z) - \log q(z|x^{(i)},y^{(i)},t^{(i)}) \right],
$
where $q(z|x,y,t)$ is an approximate variational posterior. 
We optimize the above objective using gradient descent, applying the reparameterization trick to estimate the gradient.
We compute the counterfactual $Y[x,t]$ using auxiliary inference networks as in earlier work \citep{louizos2017causal}. See Appendix \ref{apdx-architecture} for the full derivation. 


\subsection{Properties of Deep Multi-Modal Structural Equations}

\paragraph{Recovering Causal Effects}
The DMSE and DGSE models determine the true causal effect when their causal graph is correct and they recover the true data distribution. The following argument is analogous to that made in most previous works on causal deep generative models \cite{louizos2017causal, pengzhou2021principled, pawlowski2020deep}.

\begin{theorem}
The DMSE and DGSE models recover the true $\text{ITE}(x, \mathcal{M})$ for any subset $\mathcal{M} \subseteq\{1,2,...,m\}$ of observed modalities whenever they represent the true data distribution $p(x,y,t,z)$.
\end{theorem}
\begin{proof}
We establish the theorem for DMSEs; the proof for DGSEs is analogous with $m=1$. 
Let $x_\mathcal{M} = \{x_j \mid j \in \mathcal{M}\}$ be the data from the observed subset of modalities.
We need to show that $p(y | x_\mathcal{M}, \text{do}(t=t'))$ is identifiable for any $t'$. Observe that
\begin{align*}
   p(y | x_\mathcal{M}, \text{do}(t=t')) 
   = \int_{z} p (y | z, x_\mathcal{M}, do(t=t')) p (z | x_\mathcal{M}, do(t=t')) dz 
   = \int_{z} p (y | z, x_\mathcal{M}, t') p (z | x_\mathcal{M}) dz, 
\end{align*}
where the second equality follows from the rule of do-calculus (applying backdoor adjustment).
Since our proof holds for any $t'$ and all elements on the right-hand side are identifiable, the claim follows.
\end{proof}

Note that in practice our assumption may not hold (e.g., neural network optimization is non-convex and may fail), 
but there is evidence of both failure modes \cite{rissanen2021} as well as successful settings in which deep latent variable models provide useful causal estimates \cite{pengzhou2021principled, pawlowski2020deep, mayer2020missdeepcausal, shi2019adapting}. 
See our Discussion section for additional details.

\paragraph{Identifiability in Linear Models}

Structural equations parameterized by non-convex neural networks are less amenable to analysis than simpler model classes.
However, we may
provide theoretical guarantees in the special case where a DMSE model (Equations \ref{eqn:dmse}) is linear, i.e., each equation with input variables $u \in \mathbb{R}^{d_1}$ has the form $A \cdot u + b$ for some $A \in \mathbb{R}^{d_2 \times d_1}, b \in \mathbb{R}^d_1$. Specifically, we establish in Appendix \ref{app:theory} the following result.
\newpage
\begin{theorem}
  Given a binary treatment $t$, a univariate outcome $y$, confounder $z$ and proxy variables $u, v, w$, the causal effect $P(y|\textrm{do}(t))$ is identifiable if
  \begin{enumerate}
      \item The structural equations follow a DMSE model (Figure \ref{fig:multimodal_structure}, solid edges) and are linear.
      \item Three independent views of $z$ are available in the form of proxies $u, v, w$ such that $u \perp v \perp w | z$ and the equations between $z$ and $u, v, w$ are parameterized by matrices of rank $\text{dim}(z)$
  \end{enumerate}
\end{theorem}
Our proof extends techniques developed by Kuroki and Pearl \citep{kuroki2014measurement} to high-dimensional proxy variables. Interestingly, our result crucially relies on an independence structure specified by Figure \ref{fig:multimodal_structure} (specifically, the existence of three independent proxy variables), which lends additional support for this modeling assumption and for the development of variational techniques specialized to this model family.

\paragraph{Computing Causal Effects}
Given a subset of modalities $\mathcal{M}$, we can compute the ATE \& ITE as $\E (\text{ITE}(x, \mathcal{M}))$, where $\text{ITE}(x,\mathcal{M}) =  Y[x,t=1,\mathcal{M}] -  Y[x,t=0,\mathcal{M}]$ and
\begin{align}
\label{eqn:ite_comp}
   p(y | x, \text{do}(t=t')) &= \int_{z} p (y| t=t', z) p (z | x)dz 
                \approx \int_{z} \bigr( p (y | t=t', z) p(z)\prod_{j=1}^{m} \tilde q(z|x_j)\bigl) dz,
\end{align}
where we use our variational posterior  formulation from Equation~\ref{eqn:product_of_experts} to approximate the true posterior $p(z|x)$.

\section{Experimental Results}
\label{sec:experiments}

\subsection{Synthetic Demonstration Dataset}
We start with a demonstration that provides intuition for why proxy variables are important, and how unstructured proxies can serve in place of featurized (structured) ones. 
The following small-scale synthetic setup (\citet{louizos2017causal}) involves a data distribution $\mathbb{P}$ over binary variables $y, t, z, x$:
\begin{align*}
    \mathbb{P}(z =1) = \mathbb{P}(z=0) = 0.5
    &&
    \mathbb{P}(x=1|z=1) = \rho_{x1} = 0.3
    &&
    \mathbb{P}(x=1|z=0) = \rho_{x0} = 0.1 \\
    y = t \oplus z
    &&
    \mathbb{P}(t=1|z=1) = \rho_{t1} = 0.4
    &&
    \mathbb{P}(t=1|z=0) = \rho_{t0} = 0.2
\end{align*}
where $0 < \rho_{x1}, \rho_{x0}, \rho_{t1}, \rho_{t0} < 1$ are parameters. We also introduce an unstructured proxy variable $\mathbf{X}$ that represents an ``image version" of $x$. The variable $\mathbf{X}$ will be a random MNIST image of a zero or one, depending on whether $x=0$ or $x=1$.
Formally, $\mathbf{X}$ is distributed as follows:
\begin{align*}
    \mathbb{P}(\mathbf{X}|x=1)\text{ is unif.~over MNIST images of `1' } &&
    \mathbb{P}(\mathbf{X}|x=0)\text{ is unif.~over MNIST images of `0'}
\end{align*}
\begin{wraptable}{l}{8 cm}
\vspace{-0.3cm}
\caption{Treatment effect estimation on the synthetic demonstration dataset.}
\label{table:toyexperiment}
\vskip 0.15in
\centering
\begin{small}

\begin{tabular}{lcccr}
\toprule
& Setting & $\varepsilon_{ATE}$ (Train) & $\varepsilon_{ATE}$ (Test) \\
\midrule
     Deep Str Eqns &
     Binary & 0.062
(0.012)
 & 0.069 (0.015) \\
      &
     Image & 0.068
(0.018)
 & 0.096
(0.018)
 \\
     \midrule

     IPTW ~\citep{Lunceford2004iq}  &
     Binary & 0.090
(0.005)
 & 0.127
(0.016)
 \\
      &
     Image & 5.050
(0.607)
 & 4.067
( 0.533)
 \\
     \midrule

     Augmented  &
     Binary & 0.442 (0.040) & 0.487
 (0.037)
 \\
     IPTW~\citep{robins2000onprofile} &
     Image & 4.717
(0.670)
 & 6.426
 (1.603)
 \\
     \midrule
    
  Non-Causal  &
  Binary & 0.197 (0.003)
 & 0.206 (0.004)\\
  Baseline &
  Image & 0.214 ( 0.026)
 & 0.228 ( 0.025) \\

\bottomrule
\end{tabular}

\end{small}
\vspace{-0.2cm}
\end{wraptable}
First, this is a setup that requires us to model proxies : treating $\mathbf{X}$ as a confounder as using a model of $\mathbb{P}(y \mid \mathbf{X}, t)$ 
recovers the true ATE only when $\rho_{t1} =  1 - \rho_{t0}$ and $\rho_{x1} =  1 - \rho_{x0}$ (i.e., when $\mathbf{X}$ is perfectly informative of $z$), otherwise it fails (see also Appendix \ref{apdx-toy}).

We also show that structural equations solve this task. We sample 3000 data points from $\mathbb{P}$ and fit DGSE models to 80\% of the data points $\{x, y, t\}$ (the {\sc Binary} setting) as well as on $\{\mathbf{X}, y, t\}$ (the {\sc Image} setting). 
We note the Average Treatment Effect (ATE) on the training and test sets, and we report results in Table~\ref{table:toyexperiment}. We compare DGSE with the Inverse Probability of Treatment Weighted estimator (IPTW)~\citep{Lunceford2004iq} and the doubly robust Augmented-IPTW~\citep{robins2000onprofile}---in each case the propensity score model is an MLP trained to predict $t$ from either $x$ or $\mathbf{X}$. We found that replacing $x$ with an image $\mathbf{X}$ causes the model to output highly miscalibrated probabilities close to $0,1$ (while maintaining good accuracy), which results in large and volatile inverse propensity weights and in poor ATE estimates.



\subsection{Benchmark Datasets for Causal Effect Estimation}
\paragraph{IHDP}
The Infant Health and Development Project (IHDP) is a popular benchmark for causal inference algorithms ~\citep{hill2011bayesian} that contains the outcomes of comprehensive early interventions for premature, low birth weight infants. We create a benchmark for multi-modal causal inference based on IHDP in which we can replace certain features with their "unstructured version".
%
We choose 9 of the 25 features available in IHDP in order to magnify their relative importance and accurately measure the effects of their removal. 
Please refer to Appendix~\ref{apdx:ihdp-star-setups} and~\ref{apdx:evaluating-embeddings} for detailed setup. 

\begin{wraptable}{r}{6.5cm}
\vspace{-0.5cm}
\caption{Multimodal Experiments on IHDP Dataset:  With deep structural equations, replacing baby's gender with corresponding image embedding (8 attrs + image) shows some increase in ATE error as compared to IHDP-Mini setting (9 attrs) but is better than dropping this modality altogether (8 attrs).}
\label{table:IHDP-multimodal1}
\vskip 0.15in
\centering
\tabcolsep=1pt
\begin{small}
\begin{tabular}{lccccr}
\toprule
Model  & $\varepsilon_{ATE}$ (Train+Val) & $\varepsilon_{ATE}$  error(Test) \\
\midrule
        Deep Str Eqns \\
        \midrule
        9 attrs  & 0.259 (0.037) & 0.487 (0.078) \\
        8 attrs  & 0.392 (0.141) & 0.620 (0.158)\\
        8 attrs + image  & 0.372 (0.107) & 0.575 (0.130)\\
        \midrule
        CFRNet \\
        \midrule
        9 attrs  & 0.433 (0.063) & 0.549 (0.090) \\
        8 attrs  & 0.412 (0.062) & 0.608 (0.107)\\
        8 attrs + image  & 0.501 (0.076) & 0.617 (0.114)\\
        \midrule
        OLS \\
        \midrule
        9 attrs  & 0.424 (0.061) & 0.584 (0.100) \\
        8 attrs  & 0.429 (0.066) & 0.593 (0.103)\\
        8 attrs + image  & 0.428 (0.064) & 0.590 (0.101)\\
        
 \bottomrule
\end{tabular}
\end{small}
\vspace{-0.8cm}
\end{wraptable}

\paragraph{STAR}
The Student-Teacher Achievement Ratio (STAR) experiment~\citep{DVN/SIWH9F_2008} studied the effect of class size on the performance of students.
We consider small class size as treatment; the outcome is the sum of the reading and math scores of a student. We `derandomize' this dataset by removing 80\% of the data corresponding to white students in the treated population. 
Similarly to IHDP, we select 8 attributes for the multi-modal experiment. Further details can be found in the Appendix~\ref{apdx:ihdp-star-setups} and~\ref{apdx:evaluating-embeddings}. 

\paragraph{Adding Unstructured Modalities}
We create a benchmark for multi-modal causal inference derived from IHDP and STAR in which we replace features with unstructured inputs that contain the same information as their featurized versions.
On IHDP, we replace the attribute `baby's gender' with the CLIP embedding \citep{li2022clip} of an image of a child between ages 3 to 8 years, drawn from the UTK dataset~\citep{zhang2017age}. 
On STAR, we replace the attributes corresponding to the student's ethnicity and gender by selecting an image of a child with the same ethnicity and gender from the UTK dataset. 

We train and evaluate models on datasets where the image `replaces' the attribute. (e.g., {\sc 8 attrs + image}).  We also consider two other settings for comparison: a) the original attribute is included (e.g., {\sc 9 attrs}) and b) the attribute is dropped from the reduced set of input features ({\sc 8 attrs}). 

\paragraph{Results}

As seen in Table~\ref{table:IHDP-multimodal1}, the degradation in ATE error from replacing the baby's gender by a photograph is lower as compared to removing the attribute entirely. 
This shows our models leverage signal found in the unstructured image modality with the help of deep neural networks. 

We compare these results with a simple Ordinary Least Squares model (OLS) baseline as described by~\citet{shalit2017estimating} to predict treatment effect. 
OLS shows a similar behavior when replacing baby's gender with corresponding image, however ATE errors are generally worse as compared to ATE error produced by DGSE. We also compare this with Counterfactual Regression Network (CFRNet) \citep{johansson2018learningrep} baseline. However CFRNet did not show benefits of using image modality unlike our approach.

In Table~\ref{table:STAR-multimodal1}, replacing gender and ethnicity attributes on STAR with the corresponding image improves ATE errors as compared to dropping these two attributes entirely. This shows that we can use an image to extract multiple attributes while doing causal inference. 
The CFRNet baseline shows a similar behavior, but the difference between average ATE errors across different setups is small.

\subsection{Genome-Wide Association Studies}

We evaluate our methods on an important real-world causal
inference problem---genome-wide association study analysis (GWASs).
A GWAS is a large observational study that seeks to determine the causal effects of genetic markers (or genotypes) on specific traits (known as phenotypes).
In this setting, treatment variables are high-dimensional genomic sequences, and existing propensity scoring methods (which learn a model of $p(t|x)$ for a binary $t$) are not easily applicable \citep{pryzant2017predicting,veitch2019using}---they may require training an impractical number of models. Our approach, on the other hand, can easily be used with high-dimensional $t$.

\paragraph{Background and Notation}

As motivation, consider the problem of linking a plant's genetic variants $t \in \{0,1\}^{d}$ with {\em nutritional yield}, which we model via a variable $y \in \mathbb{R}$. 
Our goal is to determine if each variant $t_j$ is {\em causal} for yield, meaning that it influences biological mechanisms which affect this phenotype \cite{crouch2020polygenic}.
We also want to leverage large amounts of unstructured data $x$ (e.g., health records, physiological data) that are often available in modern datasets \cite{bycroft2018uk,li2020electronic}.
\begin{wraptable}{r}{9 cm}
\vspace{-0.3cm}
\caption{Comparison of standard DSE methods with linear baselines. The $\ell_1$ column refers to $\|\bm{\wh\gamma} - \bm{\gamma^\star}\|_1$ where $\bm{\wh\gamma}$ is the vector of estimated causal effects, and $\bm{\gamma^\star}$ is the vector of ground truth causal effects. Precision and recall are defined in \cref{sec:gwas-simulation-appendix}. Standard error of the Mean (sem) is computed over $10$ seeds.}
    \centering
    \begin{small}
    
\begin{tabular}{lrrr}
\toprule
           Model & $\ell_1 (\downarrow)$ & Precision $(\uparrow)$ & Recall $(\uparrow)$ \\
                &     Mean (sem) &      Mean (sem) &   Mean (sem) \\
\midrule
Optimal &     0.22 (0.04) &      0.97 (0.03) &    1.0 (0.00) \\
 DSE ($2$ modalities) &     0.30 (0.06) &      0.93 (0.04) &    1.0 (0.00) \\
            LMM &     0.44 (0.06) &      0.85 (0.08) &    1.0 (0.00) \\
   DSE ($1$ modality) &     0.60 (0.09) &      0.78 (0.08) &    1.0 (0.00) \\
          PCA ($1$ component) &     0.93 (0.17) &      0.58 (0.09) &    1.0 (0.00) \\
           FA ($1$ component) &     1.08 (0.17) &      0.62 (0.08) &    1.0 (0.00) \\
          PCA ($2$ components) &     1.38 (0.24) &      0.44 (0.09) &    0.9 (0.07) \\
           FA ($2$ components) &     1.44 (0.30) &      0.55 (0.09) &    1.0 (0.00) \\
          PCA ($3$ components) &     1.66 (0.23) &      0.37 (0.08) &    0.8 (0.08) \\
           FA ($3$ components) &     1.89 (0.45) &      0.44 (0.08) &    0.9 (0.07) \\
\bottomrule
\end{tabular}

    \end{small}
    \label{table:gwas-simulation-main}
    \vspace{-0.5cm}
\end{wraptable}
\vspace{-0.3cm}

%
%

A key challenge in finding causal variants is ancestry-based confounding \cite{astle2009population,vilhjalmsson2013nature}. 
Suppose that we are doing a GWAS of plants from Countries A and B; plants in Country A get more rainfall, and thus grow faster and are more nutritious. A simple linear model of $y$ and $t$ will find that any variant that is characteristic of plants in Country A (e.g., bigger leaves to capture rain) is causal for nutritional yield.


\paragraph{Methods and Baselines}
Most existing GWAS analysis methods for estimating the effect of a variant $t_j$ rely on latent variable models:  (1) they treat all remaining variants $x$ as proxies and obtain $z$ via {\em a linear projection} (e.g., PCA \cite{price2006principal,price2010new} or LMM \cite{yu2006unified,lippert2011fast}) of $x$ into a lower dimensional space where genomes from Countries A and B tend to form distinct clusters (because plants from the same country breed and are similar); (2) we assume a {\em linear model} $\beta^\top t$ of $y$ and add $z$ into it, which effectively adds the country as a feature ($z$ reveals the cluster for each country); this allows the model to {\em regress out} the effects of ancestry and assign the correct effect to variants $t$ (one at a time).

\begin{wrapfigure}{r}{4 cm}
\centering
\vspace{-0.5 cm}
\includegraphics[scale=0.35]{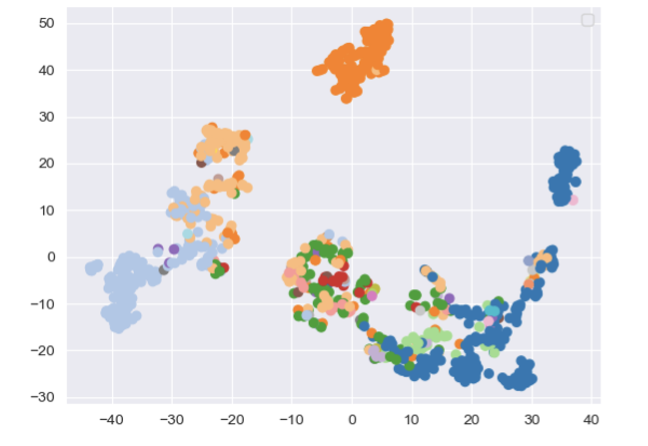}
\caption{Latent $z$ extracted by DSEs. Plants from different countries form clusters (in color).} 
\label{fig:latent_visualization}
\vspace{-0.4cm}
\end{wrapfigure}
\noindent
{\em Baselines.}
Our main baselines are Principal Component Analysis (PCA) and Linear Mixed Model (LMM), as described above and implemented via the popular LIMIX library \citep{lippert2014limix}.
We also compare against Factor Analysis (FA), a standard linear technique for deriving latent variables, Uniform Manifold Approximation and Projection (UMAP), a manifold learning technique for dimensionality reduction \cite{mcinnes2018umap} and a linear model with no correction for confounding.

\noindent
{\em DSE Models.}
We compare against deep structural equation models that leverage one or more sets of proxy variables coming from the following unstructured modalities: genomic sequences, weather time series, simulated physiological time series.
We fit DSE models via a stagewise strategy analogous to how classical GWAS models are fit: (1) we fit the component $p(z_k|x_k)$ for each proxy $x_k$; 
(2)
we fit a {\em linear model} of $y$ given $t$ and the $z$ to estimate causal effects. Thus, our $p(x_k|z)$ components are deep, while $p(y|x,z)$ is shallow (following standard assumptions on epistasis in GWAS).


\subsubsection{Simulated Human GWAS}

We used the 1000 Human Genomes \cite{Auton2015} dataset to generate a simulated multi-modal GWAS dataset, following the biologically-inspired ``Spatial'' strategy studied in \citep{tran2017implicit}.
In addition to genotypes, we generated random physiological time series by sampling Fourier series conditioned on the confounders.
Please see \cref{sec:gwas-simulation-appendix} for details. 

\cref{table:gwas-simulation-main} shows our results. Amongst the non-oracle baselines, multi-modal DSE has the smallest error in estimating causal effect, as well as the highest precision and recall at identifying causal SNPs.
The uni-modal DSE, while worse than LMM, still beats PCA and FA. 
Note that the LMM model is not compatible with multiple unstructured proxies.
In general, we see that precision starts to deteriorate faster than recall, suggesting that false positives are more likely from the weaker linear deconfounding methods such as PCA/FA.
Our results here further support that additional sources of unstructured multi-modal data can improve GWAS.

\subsubsection{Real-World Genomic Prediction and GWAS in Plants}
\begin{wraptable}{r}{8 cm}
 \vspace{-0.3cm}
\caption{Multimodal Experiments on IHDP and STAR Datasets comparing DMSE and DGSE methods. When the image modality is missing, DMSE still produces reasonable ATE errors as it can compute the latent representation on a subset of modalities better as compared to DMSE}
\label{table:IHDP-multimodal2}
\vskip 0.15in
\begin{center}
\begin{small}
\begin{tabular}{lccccr}
\toprule
Dataset & Model  & $\varepsilon_{ATE}$ (train+val) & $\varepsilon_{ATE}$  error(test) \\
\midrule
     IHDP  & DMSE  & 0.433  (0.057) & 0.627 (0.094) \\
       & DGSE  & 0.794  (0.308) & 1.080 (0.337)\\
 \midrule
 
      STAR & DMSE  & 32.575  (1.634) & 33.743 (1.890) \\
        & DGSE  & 59.102 (4.734) & 60.152 (4.788)\\
 \bottomrule
\end{tabular}
\end{small}
\end{center}
\vspace{-0.3cm}
\end{wraptable}
We also tested our methods on a real-world plant GWAS dataset from the 1001 Genomes Project for \emph{Arabidopsis Thaliana} plants ~\citep{kern_notebook, weigel20091001}. 
The challenge here is that we don't know true causal effects; 
therefore, we define a phenotype $y$ for which the true causal effect is zero---specifically, we set $y$ to the GDP of the country where each plant grows.
%
Because the genomes of plants from the same country are similar, there exist spurious correlations caused by latent subpopulation groups $z$. Our goal is to detect and correct for these confounding effects.
In addition to genomic data, we gather historical weather time series from the location of each plant (see Appendix~\ref{apdx:plant-gwas}) and use both modalities to correct for confounding.
\begin{wraptable}{l}{8 cm}
\vspace{-0.7cm}
\caption{Correcting for confounding on a plant GWAS dataset. DSEs can discover complex, non-linear clusters over genomic and weather data to identify the latent confounding variable better.}
\label{table:gwas_r2}
\vskip 0.15in
\begin{center}
\begin{small}
\begin{tabular}{lccccr}
\toprule
Model  & Input & $R^2 (\downarrow)$  \\
\midrule
        Deep Str Eqns (Ours) & Weather+SNPs & 0.049 (0.022) \\
        PCA  & Weather+SNPs & 0.097 (0.019)  \\
        UMAP  & Weather+SNPs & 0.412 (0.026)  \\
\midrule
        Deep Str Eqns (Ours) & Weather & 0.555 (0.027) \\ 
        PCA  & Weather & 0.545 (0.017)  \\ 
        UMAP  & Weather & 0.653 (0.013) \\
\midrule
        Deep Str Eqns (Ours) & SNPs & 0.068 (0.029) \\ 
        PCA  & SNPs & 0.130 (0.020)  \\
        UMAP  & SNPs & 0.406 (0.020) \\
        LMM & SNPs & 0.804 (0.009)\\
        Linear Model & - &  0.665 (0.012) \\
 \bottomrule
\end{tabular}
\vspace{-0.3cm}
\end{small}
\end{center}
\end{wraptable}

We evaluate whether each model learned
to correctly account for latent confounding effects by measuring the predictive power of $\beta^\top t$ for
$y$, where $\beta$ are the causal effects and $t$ is the vector of variants. 
Specifically, we compare the $R^2$ correlation between $\beta^\top t$ and $y$---here, {\em lower is better}, since a model that has learned the causal effects should not be predictive of the phenotype.
%
In Table~\ref{table:gwas_r2}, we see the effect of extracting confounding variables using DSEs as opposed to using the standard PCA technique. We can see that the $R^2$ values produced using DSEs are closer to $0$ as compared to using PCA. 
%
This experiment shows that neural network architectures are effective in dealing with unstructured genomic and weather data while correcting for confounding. Please refer to Appendix~\ref{apdx:plant-gwas} for details.


\subsection{Multimodal Experiments With Missing Modalities}

We demonstrate the ability of the DMSE model to handle missing data. Our IHDP and STAR benchmarks involve two modalities: images (e.g., baby's gender in IHDP) and tabular data (e.g., the remaining features). 
We compare DMSE and DGSE models on these datasets when some of the modalities may be missing. 

For DMSE, we define two different modalities $X_1$ and $X_2$ for the tabular and image modalities respectively. For DGSE, we concatenate the image embedding to the tabular modality while training the network. We evaluate ATE while randomly dropping 50\% of the images. DMSE utilizes its product-of-experts inference network to approximate the  posterior distribution when modalities are missing. DGSE cannot do this, and we resort to feeding it a vector of zeros when an image is missing. Table~\ref{table:IHDP-multimodal2} shows that DMSE produces improved ATE estimates as compared to DGSE.
\section{Related Work}


\paragraph{Multi-Modal Causal Inference}

While previous work analyzed unstructured interventions $t$ consisting of natural language \cite{pryzant2017predicting,pryzant2018interpretable,pryzant2020causal} (e.g., determining the effect of a polite vs.~a rude response) as well as unstructured $y$ \cite{batmanghelich2016probabilistic,pawlowski2020deep} (e.g., MRI images), our work proposes methods to handle unstructured $x$.
\citet{veitch2019using} developed models that correct for confounders from a single unstructured proxy $x$ derived from text \cite{egami2018make,veitch2020adapting} or a graph \cite{veitch2019using}. These approaches rely on a propensity scoring framework---they train a discriminative model of $p(t|x)$---hence do not support proxy variables or missing data, and require pre-trained text embeddings. 
Additionally, propensity scoring methods rely on neural approximators for $p(t|x)$, which may output volatile probabilistic outputs that lead to unreliable effect estimates \citep{kallus2020deepmatch}.
Our method (i) works across all modalities (beyond text or graphs), (ii) supports arbitrary numbers of proxies, (iii) supports missing data by virtue of being generative.

\paragraph{Deep Latent Variable Models}
Representation learning in causal inference has been studied by \citet{johansson2018learningrep, johansson2018learningweighted, johansson2019support} and \citet{DBLP:journals/corr/abs-2102-11107}. 
Deep latent variable models find applications throughout causal inference \cite{louizos2017causal, mayer2020missdeepcausal, pengzhou2021principled, pawlowski2020deep, zhang2020tedisentangled, vowels2020targeted, hyemi2020counterfactualfairness}.
\citet{pawlowski2020deep} study unstructured outcomes $y$ (MRI scans), but do not support proxies.
\citet{louizos2017causal} use variational auto-encoders to estimate confounders from proxies; we introduce a more structured model that handles multiple proxies that can be missing, and obviates the need for auxiliary modules. 
\citet{tran2017implicit} propose implicit deep structural equations for GWAS; ours are explicit and thus easier to train.




\section{Discussion}

\paragraph{Identifiability}
Approaches to causal effect estimation that rely on deep learning \cite{louizos2017causal, shalit2017estimating, pawlowski2020deep, shi2019adapting, mayer2020missdeepcausal, hartford2016deepcounterfactual, zhang2020tedisentangled, Yoon2018GANITEEO} can never guarantee the recovery of causal effects---neural network optimization is itself non-convex and has no guarantees. Other failure modes of deep latent variable models (DLVMs) include potentially not having a sufficiently expressive model, not having enough data to learn the model, as well as shortcomings of approximate inference algorithms. That said, there is ample evidence of both failures \cite{rissanen2021} and successes of DLVMs \cite{pengzhou2021principled, pawlowski2020deep, mayer2020missdeepcausal}. 
The DLVM approach is appealing over existing propensity scoring methods \citep{pryzant2017predicting,veitch2019using} because: (i) it naturally handles unstructured proxies that may be missing at random; (ii) it supports high-dimensional treatment variables in settings like GWAS, where propensity scoring algorithms are not easily applicable. \citet{rissanen2021} empirically identify multiple failure modes of DLVMs; our work and that of others identifies success cases (particularly in GWAS \citep{tran2017implicit,wang2019blessings}), and ultimately the validation of DLVM methods needs to be empirical \cite{shi2019adapting, Yoon2018GANITEEO, veitch2019using, hartford2016deepcounterfactual, hyemi2020counterfactualfairness}.

\paragraph{Missing Data}
We make the common assumption that data is missing at random (MAR). This poses challenges if, for example, patients missing outcomes are ones that are more likely to be sick. When two modalities and their missingness are correlated, their $x_j, x_k$ nodes in Figure \ref{fig:multimodal_structure} could be merged, somewhat addessing the issue. We leave the full exploration of non-MAR models to future work.

\section{Conclusion}

In conclusion, we proposed an approach based on deep structural equations that can leverage useful signal present in unstructured data to improve the accuracy of causal effect estimation. 
Unlike previous methods that relied on propensity scores \citep{pryzant2017predicting,veitch2019using}, ours does not suffer from instabilities caused by volatile predictive probabilities coming out of neural networks, naturally handle missing data, and are applicable in settings in which the treatment variable is high-dimensional (such as in genome-wide association studies).
Our work highlights the benefits of using large amounts of ``dark" data that were previously left unused by existing methods to improve the accuracy of causal effect estimation.

\section*{Acknowledgements}
This work was supported by Tata Consulting Services, the Cornell Initiative for Digital Agriculture, and an NSF CAREER grant ($\#2145577$).


\bibliography{references}
\bibliographystyle{neurips_style/bib_style}

\newpage
\appendix
\section{Details of the toy experiment.}
\label{apdx-toy}
\begin{enumerate}
\item \textbf{Dataset.} In this experiment, we generate synthetic dataset consisting of 5 variables: a latent binary variable $Z$, a binary variable $X^*$, an MNIST image $\mathbf{X}$, and two other binary variables $y, t$ standing for outcome and treatment respectively. 
\begin{align*}
& Z_i, X^*_i \in \{0,1\} \\
& \mathbf{X}_i \in \R^{28 * 28}   \text{~~ is an MNIST image of 0 or 1} \\
& y_i, t_i \in \{0,1\} \\
& i \in\{ 1, 2, \cdots, N\}, N = 3000
\end{align*}

These variables are sampled from the following distribution:
\begin{align*}
& P[Z = 0] = P[Z = 1] = 0.5 \\
& P[t = 1 | Z = 0]= \rho_{t0}, P[t = 1 | Z = 1] = \rho_{t1} \\
& P[X^* = 1 | Z = 0]=  \rho_{x0},P[X^* = 1 | Z = 1] = \rho_{x1} \\
& \mathbf{X} = \begin{cases}
\text{Binarized image of 0 when } X^* = 0 \\
\text{Binarized image of 1 when } X^* = 1 
\end{cases} \\
& y = t \oplus Z
\end{align*}

\item \textbf{Accounting for Confounding in ATE Computation. }

As mentioned before, the definition of the Average Treatment Effect (ATE) is as follows:
\begin{align*}
& ATE = \mathbb{E}[ITE(x)] \\
=~& \mathbb{E} \Big[ 
	\mathbb{E}[ {y} ~|~ \mathbf{X}=x, do(t = 1)]
	-
	\mathbb{E}[ {y} ~|~ \mathbf{X}=x, do(t = 0)]
\Big] 
\end{align*}
Now we consider the first term within the outer expectation:
\begin{align*}
~& \mathbb{E}[~y~ |~ \mathbf{X} = x,~ do(t = 1)] \\
=~& 1 * P[y = 1 | \mathbf{X} = x, do(t = 1)] + 0 \\
=~& \sum_z P[y = 1 | \mathbf{X} = x, do(t = 1), Z = z] \cdot P[Z = z | \mathbf{X} = x, do(t = 1)] \\
=~& \sum_z P[y = 1 | \mathbf{X} = x, t = 1, Z = z] \cdot P[Z = z | \mathbf{X} = x] \\
=~& P[y = 1 | \mathbf{X} = x,t = 1, Z = 0] \cdot P[Z = 0 | \mathbf{X} = x] + 0 \\
=~& P[Z = 0 | \mathbf{X} = x]
\end{align*}
where the second last equation is due to the fact that $y = t \oplus Z$. Similarly, we can compute the second term as
\[
\mathbb{E}[~y~ |~ \mathbf{X} = x, do(t = 0)] = P[Z = 1 | \mathbf{X} = x].
\]
Using the predefined generative process for this dataset, we can also write 
\begin{align*}
P[Z = z | \mathbf{X} = x] =~& \frac{P[Z = z] \cdot P[\mathbf{X} = x | Z = z]}{P[\mathbf{X} = x]} \\
=~& \frac{0.5 \cdot P[\mathbf{X} = x | Z = z] }{P[\mathbf{X} = x]}  \\
=~&  \frac{0.5  \cdot P[X^* = x^* | Z = z]}{P[\mathbf{X} = x]} 
\end{align*}
where $x$ is an image of binary variable $x^*$. Plugging in the previous results, we can compute the individual treatment effect (ITE):
\begin{align*}
& ITE(x) \\
=~& \mathbb{E}[y | \mathbf{X} = x, do(t = 1)] - \mathbb{E}[y | \mathbf{X} = x, do(t = 0)] \\
=~& \frac{P[Z = 0 | \mathbf{X} = x] -  P[Z = 1 | \mathbf{X} = x]}{P[\mathbf{X} = x]}  \\
=~& \frac{P[\mathbf{X} = x | Z = 0] -  P[\mathbf{X} = x | Z = 1]}{P[\mathbf{X} = x]}  \\
=~& \frac{P[X^* = x^* | Z = 0] - P[X^* = x^* | Z = 1]}{P[\mathbf{X} = x]}  \\
=~& \begin{cases} 0.5 \cdot ({\rho_{x1} - \rho_{x0}})/({P[X^* = 0]}) , \text{ if } x = \text{image of } 0 \\
0.5 \cdot (\rho_{x0} - \rho_{x1})/(P[X^* = 1]) , \text{ if } x = \text{image of } 1 \end{cases}
\end{align*}
Therefore we can plug in the previous equation and get the final result of ATE:
\begin{align*}
ATE =~&  \mathbb{E}[ITE(x)]  \\
=~& \sum_x P[\mathbf{X} = x] \cdot ITE(x) \\
=~& \sum_{x^*} P[X^* = x^*] \cdot ITE(\text{image of } x^*) \\
= ~&0.5 \cdot ((\rho_{x0} - \rho_{x1}) + (\rho_{x1} - \rho_{x0})) \\
=~& 0
\end{align*}
\item \textbf{ATE Computation with Non-Causal Model.} ATE computation goes wrong when $\mathbf{X}$ is taken to be the only confounder. In the following computation of $\mathbb{E}[~y~ |~ \mathbf{X} = x,~ do(t = 1)]$, the second step goes wrong as adjustment is done over $\mathbf{X}$ instead of $Z$.  
\begin{align*}
~& \mathbb{E}[~y~ |~ \mathbf{X} = x,~ do(t = 1)] \\
=~& 1 * P[y = 1 | \mathbf{X} = x, do(t = 1)] + 0 \\
=~& \sum_x P[y = 1 | \mathbf{X} = x, t = 1] \cdot P[\mathbf{X} = x] \\
=~& P(\mathbf{X} = \textrm{image of } 1) \cdot  P[y = 1 | \mathbf{X} = \textrm{image of } 1, t = 1] +\\
&P(\mathbf{X} = \textrm{image of } 0) \cdot  P[y = 1 | \mathbf{X} = \textrm{image of } 0, t = 1] \\
=~& P(X^* = 1) \cdot  P[y = 1 | X^* = 1, t = 1] +P(X^* = 0) \cdot  P[y = 1 |X^* = 0, t = 1] \\
=~& P(X^* = 1) \cdot  \frac{P[y = 1, X^* = 1, t = 1]}{ P[X^* = 1, t = 1]} +P(X^* = 0) \cdot \frac{P[y = 1, X^* = 0, t = 1]}{ P[X^* = 0, t = 1]}  \\
=~& P(X^* = 1) \cdot  \frac{ \sum_z(P[y = 1, X^* = 1, t = 1| Z=z]P(Z=z))}{ \sum_{z} (P[X^* = 1, t = 1| Z=z]P(Z=z))} \\
&+P(X^* = 0) \cdot \frac{\sum_{z}(P[y = 1, X^* = 0, t = 1| Z=z]P(Z=z))}{ \sum_{z} (P[X^* = 0, t = 1| Z=z]P(Z=z))}  \\
=~& P(X^* = 1) \cdot  \frac{ (P[y = 1, X^* = 1, t = 1| Z=0] \cdot 0.5)}{ \sum_{z} (P[X^* = 1, t = 1| Z=z]\cdot 0.5)} \\
&+P(X^* = 0) \cdot \frac{(P[y = 1, X^* = 0, t = 1| Z=0] \cdot 0.5 )}{ \sum_{z} (P[X^* = 0, t = 1| Z=z] \cdot 0.5)}  \\
=~& P(X^* = 1) \cdot  \frac{ \rho_{x0} \rho_{t0}}{ \rho_{x1} \rho_{t1} + \rho_{x0} \rho_{t0}} + P(X^* = 0) \cdot \frac{((1-\rho_{x0})\rho_{t0} )}{ (1-\rho_{x1}) \rho_{t1} + (1-\rho_{x0}) \rho_{t0}}
\end{align*}

Similarly, we get wrong expectation when performing intervention on variable $t$ to set it to zero.  
\begin{align*} \mathbb{E}[~y~ |~ \mathbf{X} = x,~ do(t = 0)]
~& = P(X^* = 1) \cdot  \frac{ \rho_{x1} (1-\rho_{t1})}{ \rho_{x1} (1-\rho_{t1}) + \rho_{x0} (1-\rho_{t0})} +\\
&P(X^* = 0) \cdot \frac{((1-\rho_{x1}) (1-\rho_{t1}) )}{ (1-\rho_{x1}) (1-\rho_{t1}) + (1-\rho_{x0}) (1-\rho_{t0})}
\end{align*}
  Also, 
  \begin{align*} P(X^* = 1)
~& = \sum_{z} P(X^* = 1|Z=z) P(Z=z)\\
~& = (\rho_{x1} + (1-\rho_{x0})) \cdot 0.5.\\
\end{align*}
\begin{align*}  P(X^* = 0) = 1 - P(X^* = 1) = (\rho_{x0}-\rho_{x1}+1) \cdot 0.5.\end{align*}

With the above expressions, we can check that setting $\rho_{x0} = 0.1$, $\rho_{t0}=0.2$, $ \rho_{x1}=0.3$, $\rho_{t1}=0.4$ gives a non-zero ATE when computed using non-causal methods (i.e. without accounting for hidden confounder). This corresponds to the non-causal baseline in Table~\ref{table:toyexperiment}. To get ATE of zero using non-causal baseline, we need to set $\rho_{t1} = 1- \rho_{t0}$ and $\rho_{x1} = 1 - \rho_{x0}$. 
\end{enumerate}

\section{Neural Architecture of Deep Structural Equations and Approximate Inference Networks}
\label{apdx-architecture}


\paragraph{Our Architecture}

In this section, we add the details of the DGSE and DMSE architecture that we used. $\mathbf{X}_i$ denotes an input datapoint, i.e. the feature vector (possibly containing multiple modalities), $t_i$
is the treatment assignment, $y_i$ denotes the corresponding outcome and $\mathbf{Z}_i$ is the latent hidden confounder. Within DGSE and DMSE, the latent variable is modeled as a Gaussian. For DGSE, we write (similar to \citet{louizos2017causal}):
\begin{align*}
& p[\mathbf{Z}_i] = \prod_{j = 1}^{D_z} \mathcal{N}(Z_{ij} ~|~ 0, 1)  \\
& p[t_i ~|~ \mathbf{Z}_i] = \mathsf{Bern}(\sigma(\text{NN}_1(\mathbf{Z}_i))) \\
& p[\mathbf{X}_i ~|~ \mathbf{Z}_i] = \prod_{j = 1}^{D_x} p[X_{ij} ~|~ \mathbf{Z}_i]  
\end{align*}
where $\sigma(\cdot)$ is the sigmoid function, $\mathsf{Bern}$ is the Bernoulli distribution,  $D_x, D_z$ are the dimensions of $\mathbf{X}$ and $\mathbf{Z}$ respectively, and $p[X_{ij} ~|~ \mathbf{Z}_i]$ is an appropriate probability distribution for the covariate $j$. If the treatment variable is not binary, we can modify the distribution appropriately. Within DMSE, it is possible to further factorize the distribution $p(X_i|Z_i)$ into product of distributions over component modalities owing to the conditional independence.

If the outcome $y$ is discrete, we parameterize its probability distribution as a Bernoulli distribution:
\begin{align*}
p[y_i ~|~ t_i, \mathbf{Z}_i] =~& \mathsf{Bern}\left(\pi = \hat{\pi_i} \right) \\
\hat{\pi_i} =~& \sigma\left( \text{NN}_2(\mathbf{Z}_i, t_i)\right)
\end{align*}
and if it is continuous, we parameterize its distribution as a Gaussian with a fixed variance $\hat{v}$, defined as:
\begin{align*}
p[y_i ~|~ t_i, \mathbf{Z}_i] =~& \mathcal{N}\left(\mu = \hat{\mu_i}, \sigma^2 =  \hat{v} \right) \\
\hat{\mu_i} =~& \text{NN}_2(\mathbf{Z}_i, t_i).
\end{align*}
Here each of the $\text{NN}_i(\cdot)$ is a neural network.

The posterior distribution for DGSE is approximated as
\begin{align*}
q[\mathbf{Z}_i ~|~ \mathbf{X_i}, t_i, y_i] =~& \prod_{j = 1}^{D_z} q[Z_{ij} ~|~ \mathbf{X_i}, t_i, y_i]  
=  \prod_{j = 1}^{D_z} \mathcal{N}(\mu_{ij}, \sigma_{ij}^2), 
\end{align*}
where
\begin{align*}
& \mu_{ij}, \sigma_{ij}^2 = \text{NN}_4 (\mathbf{X}_i, y_i, t_i).
\end{align*}
 For DMSE, the posterior distribution is computed differently using Product-of-Experts (PoE) \cite{wu2018multimodal} formulation, due to which it can handle missing modalities during training and inference gracefully.  

The objective of DGSE model is the variational lower bound defined as:
\begin{align*}
\mathcal{L} = & 
  \sum_{i = 1}^N
  \mathbb{E}_{q[\mathbf{Z}_i ~|~ \mathbf{X}_i, t_i, y_i] }
  \Big[
  \log p[\mathbf{Z}_i]
  +
  \log p[\mathbf{X}_i, t_i ~|~ \mathbf{Z}_i] \\
  & +
  \log p[y_i ~|~ \mathbf{Z}_i, t_i]
  -
  \log q[\mathbf{Z}_i ~|~ \mathbf{X}_i, t_i, y_i]
  \Big]
\end{align*}

The DMSE, on the other hand, requires a sub-sampled training objective to ensure that the modality specific posterior networks are trained and the relationships between individual modalities is captured. For DGSE, we also define the auxiliary encoders and the extra term in the variational lower bound following \citet{louizos2017causal}.

Auxiliary Encoders:
\begin{align*}
q[t_i ~|~ \mathbf{X}_i] =~& \mathsf{Bern}(\pi = \sigma(\text{NN}_5(t_i)))
\end{align*}
For discrete $y_i$, we have
\begin{align*}
q[y_i ~|~ t_i, \mathbf{X}_i] =~& \mathsf{Bern}\left(\pi = \hat{\pi_i} \right) \\
\hat{\pi_i} =~& \sigma (\text{NN}_6(\mathbf{X}_i, t_i)).
\end{align*}
For continuous $y_i$, we write 
\begin{align*}
p[y_i ~|~ t_i, \mathbf{X}_i] =~& \mathcal{N}\left(\mu = \bar{\mu_i}, \sigma^2 =  \hat{v} \right) \\
\hat{\mu_i} =~& \text{NN}_6(\mathbf{X}_i, t_i).
\end{align*}

This introduces the following extra term in the  variational lower bound:
\[
  \mathcal{L'} = 
  \sum_{i = 1}^N
  \log q[t_i ~|~ \mathbf{X}_i]
  +
  \log q[y_i ~|~ \mathbf{X}_i, t_i]
  \Big]
\]
DMSE does not involve these extra terms within its ELBO objective. 

Compared with \citet{louizos2017causal}, we can extend DGSE to different types of architectures for the posterior distribution $q[\mathbf{Z}_i ~|~ \mathbf{X_i}, t_i, y_i]$. When $\mathbf{X}$ is an image (e.g. medical scans, patient photos), we can use a suitable Convolutional Neural Network (CNN) architecture for extracting information effectively~\citep{lecun1995convolutional}. In our experiments with image modality, we used pretrained CLIP embeddings~\cite{li2022clip} in the first layer to extract relevant features from the images. To avoid the overwhelming difference between the image and two binary variables $t, y$, we also apply dimension reduction techniques such as Principle Component Analysis to the embeddings of the image before feeding it into the network that is shared with $t, y$. When $\mathbf{X}$ is time-series data, (e.g. text, recording), we can change the architecture to recurrent neural networks such as Long Short Term Memory \citep{hochreiter1997long}. More generally, we can choose modality specific architectures and make appropriate design choices to perform learning and inference over unstructured modalities as inputs. DMSE can handle different types and lengths of modalities gracefully and also work with missing modalities owing the specialized variational learning and inference procedures.  
\section{Comparing Our Methods with Other VAE- Based Estimators}

While our method is an instance of generative models, we identify the following key differences:

\begin{enumerate}
    
\item We propose \textbf{new generative model architectures} that extend existing models (e.g., DSE, CEVAE) to multiple proxies $X_i$, each possibly coming from a different modality.
\item  We derive \textbf{novel inference algorithms} for these extended models, which have the following benefits:
\begin{enumerate}
    \item Our algorithms scale better to large sets of modalities by leveraging the independence structure of the $X_i$.
    \item Our inference algorithms naturally handle missing $X_i$.
    \item They are also simpler: they don’t require auxiliary networks (e.g., like in CEVAE~\citep{louizos2017causal}).
\end{enumerate}

\item Lastly, our key contribution is that we demonstrate the effectiveness of generative models at \textbf{modeling unstructured proxies} (many previous methods instead relied on propensity scoring).
\end{enumerate}

Appendix~\ref{apdx:additional_cevae_comparison}  empirically shows that DMSE model compares favorably against CEVAE on synthetic datasets. 
\section{Setups used for IHDP and STAR Dataset experiments}
\label{apdx:ihdp-star-setups}
\subsection{IHDP Experiments}
The data corresponding to non-white mothers in the treated set of children is removed so that causal effect of the intervention cannot be estimated directly. The column corresponding to mother's race is removed so that this confounder cannot be obtained directly from the input. We consider 100 replicates of this dataset, where the output is simulated according to setting 'A' of NPCI package \citep{npci}. The true treatment effect is known as the simulation provides expected output values for both values of binary treatment variable. We train a DGSE model on each replicate with a 63/27/10 ratio of training, validation and test dataset size. We set the latent dimension to be 20 units and the number of hidden layers to be 2. The hidden layers have size of 20 units.

\paragraph{The IHDP-Full Setting} There are 25 input features in this experimental setting.
 We report the absolute error in ATE produced by DGSE and OLS for this setting in Table~\ref{table:IHDP-base}. 
\paragraph{The IHDP-Mini Setting} Here, we choose 9 features from the 25 input features so that removal of the feature `baby's gender' produces statistically significant treatment effect. We used mutual information and F-statistics between each of the original 25 features and the target variable $y$ to assess the importance of each feature in the initial 100 replicates of IHDP. While making sure that the absolute ATE errors don't deviate too much from the corresponding errors produced by IHDP-Full setting, we experimented with several combinations of the high ranking features to select the following 9 features in the IHDP-Mini setting. 

\begin{enumerate}
    \item Feature 6: `sex of baby'
    \item Feature 0: `birth-weight' 
    \item Feature 1: `b.head'
    \item Feature 2: `preterm'
    \item Feature 3: `birth.o'
    \item Feature 8: `mom married?'
    \item Feature 9: `mom’s education lower than high 
school?'
    \item Feature 12: `Smoked cig during pregnancy?'
    \item Feature 20: `harlem'
\end{enumerate}
Table~\ref{table:IHDP-base} shows a comparison of absolute ATE errors between the IHDP-Full setting and IHDP-Mini setting. Table~\ref{table:IHDP-multimodal1} shows the comparison of our approach with CFRNet~\cite{johansson2018learningrep} and Ordinary Least Squares (OLS) approach. OLS takes the concatenation of the covariates and treatment variable value as input to produce output. We see that the CFRNet baseline does not utilize the image effectively while OLS shows small difference between the setting where baby's gender was removed (8 attrs) and the setting where baby's gender was retained (9 attrs) setting. Hence the replacement of image (8 attrs+image) produces a small average improvement as compared to the setting where baby's gender was dropped (8 attrs). We also note that we have dropped 6 replicates from the 100 IHDP replicates under consideration. These 6 replicates showed a large degradation in ATE estimates by adding baby's gender (9 attrs) as compared to removing it (8 attrs). 

\subsection{STAR Experiments}
We `derandomize' this dataset by removing 80\% of the data corresponding to white students in the treated population. The dataset~\citep{DVN/SIWH9F_2008} has 15 input attributes. The true treatment effect can be estimated directly since the original dataset corresponds to a randomized controlled trial. Hence, it is possible to compute ATE error as the absolute difference between true ATE estimate and the ATE as predicted by the model. Similar to IHDP, we choose the following subset of attributes for the STAR experiment
\begin{enumerate}
    \item Feature 2: `Student grade'
    \item Feature 3: `Student class-type' 
    \item Feature 4: `Highest degree obtained by teacher'
    \item Feature 5: `Career ladder position of teacher'
    \item Feature 6: `Number of years of experience of teacher'
    \item Feature 7: `Teacher's race'
    \item Feature 10: `Student's gender'
    \item Feature 11: `Student's ethnicity'
 
\end{enumerate}

\begin{table}[h]
\caption{Multimodal Experiments on STAR Dataset: Removing student gender and ethnicity (6 attrs) shows increased ATE errors when compared with retaining these attributes (8 attrs), signaling that these two attributes are important for predicting treatment effect. Replacing these with image of a child shows no degradation in ATE estimation.}
\hspace{0.1cm}

\centering
\begin{tabular}{lccccr}
\toprule
Setting & $\varepsilon_{ATE}$ (Train+Val) & $\varepsilon_{ATE}$  error(Test) \\
\midrule
 Deep Str Eqns \\
\midrule
        8 attrs   & 36.479 (1.770) & 34.039 (2.336) \\
        6 attrs  & 43.682 (1.520) & 40.651 (2.425)\\
        6 attrs + image  & 35.581 (1.723) & 33.654 (2.476)\\
\midrule
CFRNet \\
\midrule
        8 attrs  & 61.835 (1.025) & 25.436 (2.332) \\
        6 attrs  & 62.055 (1.001) & 25.649 (2.339)\\
        6 attrs + image  & 61.350 (1.109) & 25.219 (2.313)\\
        
 \bottomrule

\end{tabular}

\label{table:STAR-multimodal1}
\end{table}

In Table~\ref{table:STAR-multimodal1}, we repeat the experiment 100 times and report average ATE errors along with standard error. We removed 8 repetitions in the experiment where DGSE or CFRNet showed lack of convergence as evidenced by very high validation loss on any setting of input attributes.

\begin{table}[h]
\caption{Treatment effects on IHDP Dataset. Using a reduced set of features in the IHDP Mini setting produces comparable absolute ATE errors as the degradation is small. Numbers in round braces indicate standard deviations. Since this is a simulated dataset, we can directly compute the treatment effect using the simulated factual and counterfactual outputs. ATE error is the absolute difference between true ATE and predicted ATE. 
}
\label{table:IHDP-base}
\vskip 0.15in
\begin{center}
\tabcolsep=1.5pt
\begin{small}
\begin{sc}
\begin{tabular}{lccccr}
\toprule
Models  & Input & $\varepsilon_{ATE}$ (train+val) & $\varepsilon_{ATE}$  error(test) \\
\midrule
       
        DGSE & Full  & 0.289 (0.027) & 0.358 (0.041) \\
        OLS & Full  & 0.535 (0.089) & 0.718 (0.132)\\
        DGSE & Mini  & 0.404 (0.107) & 0.720 (0.159) \\
 \bottomrule
\end{tabular}
\end{sc}
\end{small}
\end{center}
\vskip -0.1in
\end{table}

\section{Evaluating Quality of Pre-Trained Embeddings}
\label{apdx:evaluating-embeddings}
We demonstrate that our pre-trained embeddings contain useful signal by building a neural network model that predicts the gender, age and ethnicity from the CLIP embedding~\citep{li2022clip} of the corresponding image.

We build a simple neural network model that takes the CLIP embedding of an image as input and predicts the age of the person in that image. We use 5-dimensional PCA embeddings of 500 randomly chosen images of people aged 10-45 years (corresponding to the age-group in the IHDP experiment). We have an independent test dataset corresponding to 100 images chosen in a similar way. We see that the $R^2$ value for the age prediction on test dataset is 0.45. If we increase the size of PCA embedding to 50, this $R^2$ value increases to 0.58. Thus, it is possible to extract the age information from randomly chosen images using a simple neural network. 

In the above setting, we also studied the classification accuracy of separate neural networks that predict gender and ethnicity from the CLIP embeddings. We saw that gender was predicted with 94\% accuracy and ethnicity was predicted with 58\% accuracy using a 5-dimensional PCA embedding. After increasing the size of PCA embedding to 50 dimensions, the gender prediction accuracy increased to 95\% and ethnicity prediction accuracy increased to 77\%. This further supports our idea of replacing the attribute corresponding to 's baby's gender or student ethnicity/gender with an appropriate image.  
\section{Plant GWAS}
\label{apdx:plant-gwas}
\paragraph{Setup}
We apply our deep structural equations framework for correcting the effects of confounding.
We fit a DGSE model in two stages: (1) first, we only fit the model of $p(T|Z)$ using the DGSE ELBO objective; (2) then we fit $p(Y|Z,T)$ with a fixed $Z$ produced by the auxiliary model $q(Z|T)$. We found this two-step procedure to produce best results. 
The subsampled SNPs corresponding to each genome are taken as input $X$. The encoder and decoder use a single hidden layer of 256 units while a 10-dimensional latent variable $Z$ is used.  This network is optimized using ClippedAdam with learning rate of 0.01, further reduced exponentially over 20 training epochs. The confounding variable for each genome can now be computed as latent representation produced by the DSE. To measure the success of confounding correction, we compute the $R^2$ values between the true GDP of the region and the GDP output as predicted via our model and the baselines. If we have corrected for confounding, then we should get low $R^2$ values.

\paragraph{Historical Weather Data}
We used historical weather data collected from \citet{menne2012historicalclimate} to add  a new modality while performing plant GWAS. We use per day precipitation data from year 2000 collected by weather station closest in distance to the latitude/longitude coordinates of the location from which the SNPs of plant were collected.  For the locations where weather data was missing, we replaced those entries with zeros. 

\section{Simulated GWAS Experiments}\label{sec:gwas-simulation-appendix}
We provide additional details on this experiment here. 

\subsection{Data generating process}
To simulate the confounders, SNPs (genotypes), and the outcomes (phenotypes), we follow the ``Spatial'' simulation setup from Appendix D.1 \& D.2 of \citep{tran2017implicit}. Specifically, we generate random low-rank factorization of the allele frequency logits $F = \sigma^{-1}(\Gamma S)$, where $\sigma$ is sigmoid, as is common in the literature \citep{balding1995method,pritchard2000inference}. In the ``Spatial'' Setting~\citep{tran2017implicit}, the $S$ matrix is interpreted as geographic spatial positions of the individuals. For the $m$-th SNP of the $n$-th individual, we generate the SNP $X_{nm} \sim \op{Bin}(3,\sigma(F_{nm}))$. In this simulation, we considered $M=10$ SNPs with $N=10000$ individuals.

Individuals are clustered into $K=3$ groups based on their locations, and the individual's cluster is the \emph{unobserved} confounder. 
Then, the outcomes are calculated from \emph{both} the SNPs, where only $c=2$ SNPs have a non-zero causal effect, plus a confounding term that is a function of the cluster, and some i.i.d. Gaussian noise. One small deviation from \citet{tran2017implicit} is that our noise is i.i.d. Gaussian (hence, the variance does not depend on the confounder). 

We further augmented the dataset to include a time-series proxy that can help identify the unobserved spatial position of individuals.
For each cluster, we came up with some Fourier coefficients, 
\begin{align*}
    &\text{Cluster 0:  } a_0^{(0)} = 0.0, a_1^{(0)} = -1.0, a_2^{(0)} = 1.0, b_1^{(0)} = -1.0, b_2^{(0)} = 1.0
    \\&\text{Cluster 1:  } a_0^{(1)} = 1.0, a_1^{(1)} = -5.0, a_2^{(1)} = 2.0, b_1^{(1)} = -5.0, b_2^{(1)} = 2.0
    \\&\text{Cluster 2:  } a_0^{(2)} = -1.0, a_1^{(2)} = -2.0, a_2^{(2)} = 5.0, b_1^{(2)} = -2.0, b_2^{(2)} = 5.0.
\end{align*}
For each individual in the $k$-th cluster, we sampled at $N_{samples} = 50$ uniformly spaced times, across $N_{periods}=2$ periods of length $T=5$.
That is, our time series proxy consists of points $\{x_i\}_{i\in[N_{samples}]}$,
\begin{align*}
    &x_i = \frac{1}{2} a_0^{(k)} + \sum_{\ell=1}^2 a_\ell^{(k)}\cos\left( \frac{2\pi t_i(\ell-1)}{T} \right) + \sum_{\ell=1}^2 b_\ell^{(k)}\sin\left( \frac{2\pi t_i(\ell-1)}{T} \right),
    \\&t_i = \frac{T N_{periods} i}{N_{samples}}.
\end{align*}
Please see \cref{fig:visualize-clusters-gwas} for a visualization of the time-series generated for each cluster.
\begin{figure}[!h]
    \centering
    \includegraphics[width=0.6\linewidth]{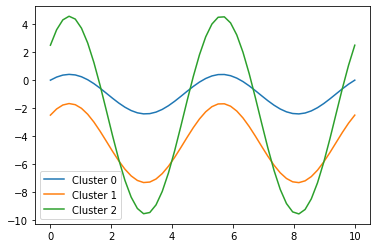}
    \caption{Visualization of Time-series Proxy.}
    \label{fig:visualize-clusters-gwas}
\end{figure}

\subsection{DSE Modeling and evaluation setup}
To handle the SNP data, we trained a Multi-Layer Perceptron (MLP) based $\beta$-variational autoencoder \cite{DBLP:conf/iclr/HigginsMPBGBML17} with ReLU activations, which output parameters of a Gaussian latent distribution. The MLPs consisted of 3 hidden layers, each with 64 units, and our latent dimension $z$ was 2. To handle time series data (new modality), we trained a $\beta$-VAE that encoded and decoded with 1D convolutions, as 1D convolutions have been shown to strongly approximate fast Fourier transforms \cite{nussbaumer65097fast}. There were 3 hidden convolutional layers, each of which had 32 output channels, with a kernel of length 3 and a stride of 1. Following the convolution was an MLP with 1 hidden layer consisting of 64 units. ReLU activations were used throughout the architecture.

In our experiments, we swept over $\beta \in \{ 0.1, 0.2, 0.5, 1.0, 1.2, 1.5 \}$. Our general hypothesis was that lower $\beta$ values would make our VAE perform better, as unlike in the standard VAE evaluation scheme, we sample from the posterior to generate latents for individuals, and not the prior, so the KL divergence term in the ELBO doesn't matter as much.

When evaluating on a set of SNPs and time series data corresponding to a set of individuals, we generate latents by passing the data of each respective type into the appropriate encoder and sampling from the resulting distributions. The latent $z$ for each individual is the concatenation of the SNP latent $z_{\textrm{snp}}$ and the time series latent $z_{\textrm{ts}}$.

Once we have the latent confounder $z$ for an individual, we calculate the causal effect of the $m^{\textrm{th}}$ SNP by performing the linear regression, 
\begin{align*}
    y = \beta_0 + \sum_i \beta_i z_i + \gamma_m x_m + \epsilon_i,
\end{align*}
where $\epsilon_i$ are i.i.d. standard normal. We solved this using the closed form solution for least-squares linear regression, given by
\begin{align*}
    \widehat{\gamma} = (A^TA)^{-1} A^T y,
\end{align*}
where $A_i = [z_i, 1, x_m]$ is the $i^{\textrm{th}}$ row in the matrix $A$. Then $\wh\gamma_m$ is our estimated causal effect for the $m^{\textrm{th}}$ SNP. We let $\bm{\wh\gamma}$ be a vector of length $M$ to denote our estimates for each SNP's causal effect.
\subsection{Results}
We tested our single modality DSE model and multi-modal DSE model against several linear latent model baselines. 
Two baselines, Principle Components Analysis (PCA) and Factor Analysis (FA), also generate a latent $z$, which we then use within our linear regression approach to calculate $\bm{\wh\gamma}$. We also ran the Linear Mixed Model (LMM) implementation by Limix~\citep{lippert2014limix}. 
We also plot two oracle baselines. The first is when $\bm{\wh\gamma} = \bm{\gamma^\star}$, labeled ``truth''. The second is when $z$ is the true confounder used when generating the data, according to ``Spatial''. 

Let $\bm{\gamma}^\star$ denote the ground truth causal effect vector. We report $\ell_1$ and $\ell_2$ norms to measure $\|\bm{\wh\gamma}-\bm{\gamma^\star}\|$.
We also report true/false positive/negatives, which we define as follows. For an individual, let $\tau = \min_{i: \gamma_i \neq 0} |\gamma_i| /2$. Then, SNP $m$ is a
\begin{enumerate}
    \item True Positive (tp): if $\wh\gamma_m \geq \tau$ and $\gamma_m^\star \geq \tau$ or $\wh\gamma_m \leq -\tau$ and $\gamma_m^\star \leq -\tau$.
    \item True Negative (tn): if $|\wh\gamma_m| < \tau$ and $|\gamma_m^\star| < \tau$.
    \item False Negative (fn): if $|\wh\gamma_m| < \tau$ but $|\gamma_m^\star| > \tau$.
    \item False Positive (fp): if none of the above hold. Concretely, there are two cases. First, if $|\wh\gamma_m| > \tau$ but $|\gamma_m^\star| < \tau$. Second, if the direction is wrong, i.e. $\wh\gamma_m > \tau$ but $\gamma_m^\star < -\tau$, or $\wh\gamma_m < -\tau$ but $\gamma_m^\star > \tau$.
\end{enumerate}
Finally, recall that precision is $tp/(tp+fp)$ and recall is $tp/(tp+fn)$, and higher is better for both.
In the following table, we show the mean and standard error of the mean (sem) over $10$ seeds at evaluation time. 
The unimodal DSE only encodes the SNPs vector into a $2$ dimensional embedding. The multimodal DSE also encodes the time-series vector into a $2$ dimensional embedding, and we concatenate this to the $2$-dimensional embedding of SNPs to form a $4$ dimensional embedding that can be used in our linear regression. $PCA_k$ denotes Principal Components Analysisn(PCA) with $k$ components. Similarly, $FA_k$ denotes Factor Analysis (FA) with $k$ components. Amongst the non-oracle baselines, multi-modal DSE has the smallest errors and highest precision/recall.  It beats all linear baselines including LMM. Unimodal DSE, while worse than LMM, still beats PCA and FA. In general, we see that precision starts to deteriorate faster than recall, suggesting that false positives are more likely from the weaker linear deconfounding techniques such as PCA/FA.
It's also interesting that when PCA/FA used more components, we see worse performance across the board.
This illustrates the importance of picking the right size of latent dimension, which is an open question. If the latent dimension is too small then we may not capture all the confounders. If the latent dimension is too large, the latent encoding may also end up capturing spurious correlations. 
\begin{table}[]
    \centering
    
\begin{small}

\begin{tabular}{lrrrrrr}

\toprule
           Model & $\ell_1 (\downarrow)$ & $\ell_2 (\downarrow)$ & tp $(\uparrow)$ & fp $(\downarrow)$ & tn $(\uparrow)$ & fn $(\downarrow)$ \\
                &     Mean (sem) &     Mean (sem) & Mean (sem) & Mean (sem) & Mean (sem) & Mean (sem) \\
\midrule
Optimal &     0.22 (0.04) & 0.09 (0.02) &  2.0 (0.00) &  0.1 (0.10) &  7.9 (0.10) &  0.0 (0.00) \\
 DSE ($2$ modalities) & 0.30 (0.06) &     0.14 (0.03) &  2.0 (0.00) &  0.2 (0.13) &  7.8 (0.13) &  0.0 (0.00) \\
            LMM &     0.44 (0.06) &     0.17 (0.02) &  2.0 (0.00) &  0.7 (0.42) &  7.3 (0.42) &  0.0 (0.00) \\
   DSE ($1$ modality) &     0.60 (0.09) &     0.26 (0.05) &  2.0 (0.00) &  1.1 (0.59) &  6.9 (0.59) &  0.0 (0.00) \\
          PCA ($1$ component) &     0.93 (0.17) &     0.40 (0.07) &  2.0 (0.00) &  2.3 (0.72) &  5.7 (0.72) &  0.0 (0.00) \\
           FA ($1$ component) &     1.08 (0.17) &     0.59 (0.13) &  2.0 (0.00) &  2.0 (0.71) &  6.0 (0.71) &  0.0 (0.00) \\
          PCA ($2$ components) &     1.38 (0.24) &     0.56 (0.09) &  1.8 (0.13) &  3.4 (0.83) &  4.6 (0.83) &  0.2 (0.13) \\
           FA ($2$ components) &     1.44 (0.30) &     0.71 (0.19) &  2.0 (0.00) &  2.6 (0.75) &  5.4 (0.75) &  0.0 (0.00) \\
          PCA ($3$ components) &     1.66 (0.23) &     0.68 (0.09) &  1.6 (0.16) &  3.8 (0.80) &  4.2 (0.80) &  0.4 (0.16) \\
           FA ($3$ components) &     1.89 (0.45) &     0.95 (0.25) &  1.8 (0.13) &  3.0 (0.56) &  5.0 (0.56) &  0.2 (0.13) \\
\bottomrule
\end{tabular}
 \end{small} 
\vspace{0.5cm}
\caption{Comparison of DSE with Baselines to Perform GWAS. }
    \label{tab:my_label}
\end{table}

\begin{table}[]
    \centering

\begin{tabular}{lrr}
\toprule
           Model & Precision $(\uparrow)$ & Recall $(\uparrow)$ \\
                &  Mean (sem) &  Mean (sem) \\
\midrule
Optimal &      0.97 (0.03) &    1.0 (0.00) \\
 DSE ($2$ modalities) &      0.93 (0.04) &    1.0 (0.00) \\
            LMM &      0.85 (0.08) &    1.0 (0.00) \\
   DSE ($1$ modality) &      0.78 (0.08) &    1.0 (0.00) \\
          PCA ($1$ component) &      0.58 (0.09) &    1.0 (0.00) \\
           FA ($1$ component) &      0.62 (0.08) &    1.0 (0.00) \\
          PCA ($2$ components) &      0.44 (0.09) &    0.9 (0.07) \\
           FA ($2$ components) &      0.55 (0.09) &    1.0 (0.00) \\
          PCA ($3$ components) &      0.37 (0.08) &    0.8 (0.08) \\
           FA ($3$ components) &      0.44 (0.08) &    0.9 (0.07) \\
\bottomrule
\end{tabular}
\vspace{0.5cm}
\caption{Comparison of DSE with Baselines in Terms of Precision and Recall Metrics. }
    \label{tab:my_label}
\end{table}
\section{Alternative Causal Graph Structures}
\label{app:alt_graphs}

This section derives estimators of ATE and ITE under assumptions on the causal graph that are different from the ones described in Figure \ref{fig:multimodal_structure}. The key takeaway of this section is that ATE and ITE can be computed under arbitrary causal graphs using minor modifications of our algorithms, which we describe in this section.

\subsection{Alternative Causal Graph Structures}

We start by formally defining the class of causal graphs that will be studied in this section. Formally, we identify two sets of modification to Figure \ref{fig:multimodal_structure}. We will derive small modifications to our algorithms to allow the new types of causal graphs implied by these modifications.

\paragraph{Observed confounders}

We consider an expanded set of graphs over a space of random variables $(X, Y, T, Z, V)$, where $V$ represents observed (non-proxy) confounders and the other random variables are associated with the observed data $x,y,t,z$. We look at causal graphs implied by structural equations of the form:
\begin{align}
\label{eqn:dmse2}
Z \sim \mathcal{P}_Z && V \sim \mathcal{P}_V &&
X_j \sim \mathcal{P}_{X_j}(\theta_{X_j}(Z)))  \;\; \forall j &&
T \sim \mathrm{Ber}(\pi_T(Z, V)) && 
Y \sim \mathcal{P}_Y(\theta_Y(Z,V,T))),
\end{align}
where $\mathcal{P}_{X_j}, \mathcal{P}_Y$ are probability distributions with a tractable density and the $\mu, \sigma, \pi, \theta$ are functions parameterized by neural networks that output the parameters of their respective probability distribution as a function of ancestor variables in the causal graph.

In the above equations, $V$ is variable that is assumed to be \textbf{always observed} (just like $y^{(i)}$ and $t^{(i)}$). All the other technical terms are defined as in the main body of the paper.
The above equations result in a causal graph with edges between $V$ and $Y,T$ and define a distribution $p(x, y, t, z, v)$.

\paragraph{Dependent proxies}

Another possible set of modifications to Figure \ref{fig:multimodal_structure} is the presence of edges between proxies $X_i, X_j$, which can be denoted as
$$
X_j \sim \mathcal{P}_{X_j}(\theta_{X_j}(Z,\text{pa}(X_j))))  \;\; \forall j
$$
where $\text{pa}(X_j))$ denotes the set of parents of $X_j$ among the other unstructured proxy variables $X_i$.


\subsection{A General Estimator Class}


The following Theorem shows that we can estimate ATE and ITE when the data distribution follows the structure in Figure \ref{fig:multimodal_structure}, plus the two types of modifications outlined above (observed confounders and dependent proxies).

\begin{theorem}
The true $\text{ITE}(x, \mathcal{M})$ for any subset $\mathcal{M} \subseteq\{1,2,...,m\}$ of observed modalities is identifiable when the true data distribution $p(x,y,t,z,v)$ has the causal graph structure of a DMSE model (Figure \ref{fig:multimodal_structure}) in addition to having observed confounders $v$ and possibly dependent confounders.
\end{theorem}
\begin{proof}
Let $x_\mathcal{M} = \{x_j \mid j \in \mathcal{M}\}$ be the data from the observed subset of modalities. Let $v$ be the observed proxy variable.
We need to show that $p(y | x_\mathcal{M}, v, \text{do}(t=t'))$ is identifiable for any $t'$. Observe that
\begin{align*}
   p(y | x_\mathcal{M}, v, \text{do}(t=t')) 
   & = \int_{z} p (y | z, x_\mathcal{M}, v, do(t=t')) p (z | x_\mathcal{M}, v, do(t=t')) dz \\
   & = \int_{z} p (y | z, x_\mathcal{M}, v, t') p (z | x_\mathcal{M}) dz, 
\end{align*}
where the second equality follows from the rule of do-calculus (applying backdoor adjustment).
Since our proof holds for any $t'$ and all elements on the right-hand side are identifiable, the claim follows.
\end{proof}

\subsection{Observed Confounders}
\label{observed_confounders}
Next, we derive an extension of the DMSE model to the setting in which we have observed confounders $V$. We refer to this modified model as \textbf{DMSE-V}.

As earlier, the \textbf{DMSE-V} model induces a tractable joint density $p(X,Y,T,Z, V)$, which allows us to fit its parameters
using stochastic variational inference by optimizing the evidence lower bound (ELBO) on the marginal log-likelihood $p(y^{(i)},x^{(i)},t^{(i)},v^{(i)})$ defined over an expanded dataset $\{y^{(i)},x^{(i)},t^{(i)},v^{(i)}\}_{i=1}^n$:
\begin{align} 
\text{ELBO}_X & = \sum_{i=1}^n \mathbb{E}_{q} \left[\sum_{j=1}^m\log p(x_j^{(i)}|z) + \log p(y^{(i)},t^{(i)},v^{(i)},z) - \log q(z|x^{(i)},y^{(i)},t^{(i)},v^{(i)}) \right],
\end{align}
where $p(y^{(i)},t^{(i)},v^{(i)},z) = p(y^{(i)} | t^{(i)},v^{(i)},z) p(t^{(i)}|z,v^{(i)}) p(z) p(v^{(i)})$
and $q(z|x^{(i)}, y^{(i)}, t^{(i)},v^{(i)})$ is the approximate variational posterior. 
Note that since $v^{(i)}$ \textit{is always observed} (just like $y^{(i)}$ and $t^{(i)}$), the $p(v^{(i)})$ term can be ignored.

In practice, this reduces to the vanilla DMSE model with the following modifications:
\begin{itemize}
\item The $\log p(y^{(i)} | t^{(i)},v^{(i)},z)$ term becomes additionally conditioned on $v^{(i)}$.
\item The $\log p(t^{(i)} | v^{(i)},z)$ term becomes additionally conditioned on $v^{(i)}$.
\item The approximate posterior $q(z|x^{(i)}, y^{(i)}, t^{(i)},v^{(i)})$ becomes additionally conditioned on $v^{(i)}$.
\end{itemize}
Crucially, the specialized inference algorithms derived for the DMSE model remain unchanged. Since $V$ is always observed, learning a \textbf{DMSE-V} model is equivalent to learning a model $p(y,x,t,z|v)$, which has the same structure as a DMSE model. In particular, all the $X_i$ are conditionally independent given $Z$. Hence, the same learning and inference algorithms apply.

\subsection{Causal Links Among Proxies}
\label{causal_links}
Another type of causal graph that we consider is one in which proxies $X_i$ are connect by causal edges. First, we note that when proxies are unstructured, such causal edges are expected to be rare, i.e., we do not expect the pixels of an image $X_i$ to have a direct influence on other variables.

When $X_i$ takes on a structured form and directly influences other proxies, our strategy is to ``collapse" any sets of variables $X_j$ that have edges among them, until we have the conditionally independent structure in Figure \ref{fig:multimodal_structure}. In the extreme case, we might need to collapse all proxies $X_i$ into a single proxy $X$ that consists of their concatenation.

The result is a model that has the same structure as DMSE, and that can be learned using the same set of inference and learning algorithms. The only drawback is increased computational efficiency.
\subsection{Additional Synthetic Data Experiment Details}

We generate synthetic data to simulate a process with modified causal links as covered in subsection~\ref{observed_confounders} and~\ref{causal_links}. Specifically, we simulate four different datasets with generative process detailed in the following section. We generate $\{X_i, t_i, y_i\}_{i=0}^{i=m}$ where $m=10000$. We use train/val/test split of 63/27/10. Here, $\oplus$ stands for logical XOR. 

\begin{enumerate}

\item \textbf{Case of Original Causal Graph (Dataset A)}

Variables: $X1$ is unstructured input, $X2$ is structured input, $T$ is treatment, $Y$ is output, $Z$ is confounder 
    
Edges in graph: $Z\rightarrow X1, Z\rightarrow X2$,  $\{ Z\}\rightarrow T, \{T, Z\} \rightarrow Y$

Thus we added no extra edges.

Generative process:
\begin{enumerate}
    \item $P(z=1)=0.5$
 \item $P(x1’=1|z=1)=P(x1’=0|z=0)=0.1$  ($x1’$ is an intermediate variable)
 \item $P(x2=1|z=1)=P(x2=0|z=0)=0.2$  
 \item $P(t=1|z=1)=P(t=0|z=0) =0.2$
 \item $y=(z \oplus t)$
 \item $P(x1|x1’ = 1)$ is unif. over MNIST images of ‘1’,  $P(x1|x1’ = 0)$ is unif. over MNIST images of ‘0’
 \end{enumerate}
 

    \item \textbf{Case of the Observed Confounder (Dataset B)}
    
    Variables: $X1$ is unstructured input, $X2$ is structured input, $T$ is treatment, $Y$ is output, $Z$ is confounder 
    
Edges in graph: $Z\rightarrow X1, \{ X2, Z\}\rightarrow T, \{X2, T, Z\}\rightarrow Y$

Thus we added extra edges $X2 \rightarrow T$ and $X2\rightarrow Y$.

Generative process:
\begin{enumerate}
    \item $P(z=1)=0.5, P(x2=1)=0.5$
 \item $P(x1’=1|z=1)=P(x1’=0|z=0)=0.1$  ($x1’$ is an intermediate variable)
 \item $P(t=1|z=1, x2=1)=P(t=0|z=0, x2=1) =0.2;$
 
 $P(t=1|z=1, x2=0)=P(t=0|z=0, x2=0) =0.9$
 \item $y=x2 \textrm{ AND } (z \oplus t)$
 \item $P(x1|x1’ = 1)$ is unif. over MNIST images of ‘1’,  $P(x1|x1’ = 0)$ is unif. over MNIST images of ‘0’
 \end{enumerate}
 

 \item \textbf{Case of the Observed Confounder (Dataset C)}
 
 Variables: $X1$ is unstructured input, $X2$ is structured input, $T$ is treatment, $Y$ is output, confounder $Z = (Z1, Z2)$
    
Edges in graph: $Z\rightarrow X1, Z \rightarrow X1, Z-> X2, \{ X2, Z\}\rightarrow T, \{X2, T, Z\}\rightarrow Y$

Thus we added extra edges $Z\rightarrow X2, X2 \rightarrow T$ and $X2\rightarrow Y$.

Generative process:
\begin{enumerate}
    \item $P(z1=1)=0.5, P(z2=1)=0.5$
    \item $P(x1’=1|z1=1)=P(x1’=0|z1=0)=0.1$  ($x1’$ is an intermediate variable)
    \item $P(x2=1|z2=1)=P(x2=0|z2=0)=0.9$
    \item $z = z_1 \oplus z_2$
    \item $P(t=1|z=1, x2=1)=P(t=0|z=0, x2=1) =0.9;$
    $P(t=1|z=1, x2=0)=P(t=0|z=0, x2=0) =0.1$
 \item $y=x2 \oplus t \oplus (z1 \oplus z2)$
 \item $P(x1|x1’ = 1)$ is unif. over MNIST images of ‘1’,  $P(x1|x1’ = 0)$ is unif. over MNIST images of ‘0’
\end{enumerate}


\item \textbf{Case of Causal Links Among Proxies (Dataset D)}

Variables: $X1$ is structured input, $X2$ is structured input, $X3$ is unstructured input, $T$ is treatment, $Y$ is output, confounder $Z = (Z1, Z2, Z3)$ 

Edges in graph:$ Z \rightarrow X1,Z \rightarrow X2, Z \rightarrow X3, X1 \rightarrow X2, 
\{Z\} \rightarrow T, \{T, Z\} \rightarrow Y$

Thus we added extra edge: $X1 \rightarrow X2$

Data generation: 
\begin{enumerate}
    \item $P(z1=1)= P(z2=1)= P(z3=1)= 0.5$
    \item $P(x1=1|z1=1)= P(x1=0|z1=0)=0.1$
    \item $P(x2=1|z2=1, x1=0)= P(x2=0|z2=0, x1=0)=0.8$
    
    $P(x2=1|z2=1, x1=1)=P(x2=0|z2=0, x1=1)=0.2$
    \item $P(x3’=1|z3=1)= P(x3’=0|z3=0)=0.3$   ($x3’$ is an intermediate variable)
    \item $z =z_1 \oplus z_2 \oplus z_3$ 
    \item $P(t=1|z=1)=P(t=0|z=0) =0.2$
    \item $y=t \oplus z$
    \item $P(x3|x3’ = 1)$ is unif. over MNIST images of ‘1’,  $P(x3|x3’ = 0)$ is unif. over MNIST images of ‘0’

\end{enumerate}

\item \textbf{Case of Increasing Number of Proxies (Dataset E)}

Variables:  $\{X_1, X_2,.. X_m\}$ are unstructured inputs, $T$ is treatment, $Y$ is output, confounder $Z = (Z_1, Z_2,..,Z_m)$. $m$ is number of modalities 
Edges in graph: $Z_i \rightarrow X_i,  Z \rightarrow T, \{T, Z\} \rightarrow Y$

\begin{enumerate}
    \item $P(z_i=1) = \frac{i}{m}$   
    \item $P(x_i’=1|z_i=1)=P(x_i’=0|z_i=0)=1$ ($x_i’$ are  intermediate variables)
    \item $z =\oplus_{i=1}^{m} z_i$    (xor over all $z_i$)
    \item $P(t=1|z=1)=P(t=0|z=0) =0.25$
    \item $P(y=1|t=1) = \textrm{sigmoid}(3z+2), P(y=1|t=0) = \textrm{sigmoid}(3z-2)$
    \item $P(x_i|x_i’ = 1)$ is unif. over MNIST images of ‘1’,  $P(x_i|x_i’ = 0)$ is unif. over MNIST images of ‘0’
\end{enumerate}

\end{enumerate}
\begin{table}[h]
\caption{Comparison of DMSE and CEVAE with increasing number of proxies}
\label{table:increasing-modalities}
\vskip 0.15in
\begin{center}
\tabcolsep=1.5pt
\begin{sc}
\resizebox{\textwidth}{!}{\begin{tabular}{l|cc|cc|cr}
\toprule
Number of & CEVAE	&&	DMSE	&&\multicolumn{2}{c}{\% improvement made by}
\\
input modalities & &&	&& \multicolumn{2}{c}{DMSE w.r.t CEVAE}\\
\midrule
 &$\varepsilon_{ATE}$ 
 &  $\varepsilon_{ATE}$  & 
 $\varepsilon_{ATE}$  &
 $\varepsilon_{ATE}$  & 
 $\varepsilon_{ATE}$  &
 $\varepsilon_{ATE}$   \\
 & (train+val) &  (test)&  (train+val) &  (test)& (train+val) &  (test) \\
\midrule
5 & 	0.0533 (0.0165) & 	0.0663 (0.0244) & 	0.0421 (0.0045)	 & 0.0472 (0.0166) & 	21.0\% & 	28.8\%\\
10 & 	0.0381 (0.0122)	 & 0.0425 (0.0148) & 	0.0296 (0.0040)	 & 0.0334 (0.0052) & 	22.5\% & 	21.5\%\\
15 & 	0.0465 (0.0062)	 & 0.0545 (0.0112) & 	0.0350 (0.0066)	 & 0.0408 (0.0011) & 	24.7\%	 & 25.1\%\\
20 & 	0.0764 (0.0178)	 & 0.0738 (0.0164) & 	0.0407 (0.0087)	 & 0.0383 (0.0054) & 	46.6\%	 & 48.1\%\\
 \bottomrule
\end{tabular}}
\end{sc}
\end{center}
\vskip -0.1in
\end{table}

\begin{table}[h]
\caption{Comparison of DMSE with CEVAE on one modality}
\label{table:one-modality-comparison}
\vskip 0.15in
\begin{center}
\tabcolsep=1.5pt
\begin{small}
\begin{sc}
{\begin{tabular}{ccc}
\toprule
Model &	ATE error &	ATE error\\
&(train+val) &	(test)\\
\midrule
 CEVAE	& 0.0637 (0.0162) &	0.0641 (0.0178)\\
DMSE &	0.0328 (0.0040) &	0.0333 (0.0045)\\
 \bottomrule
\end{tabular}}
\end{sc}
\end{small}
\end{center}
\vskip -0.1in
\end{table}

\begin{table}[h]
\caption{Comparison of CEVAE and DMSE under alternative graph structures}
\label{table:alt-graphs-comp}
\vskip 0.15in
\begin{center}
\tabcolsep=1.5pt
\begin{sc}
\resizebox{\textwidth}{!}
{\begin{tabular}{l|cc|cr}
\toprule
Causal Graph & 	\multicolumn{2}{c}{CEVAE} &		\multicolumn{2}{c}{DMSE}\\
&ATE error &	ATE error &	ATE error &	ATE error \\
&(train+val)&	(test)&	(train+val) &	(test)\\
\midrule
Dataset A: & 	0.0636 (0.0244) & 	0.0752 (0.0276) & 	0.0335 (0.0127) & 	0.0303 (0.0140)\\
Original causal graph& & & & \\
Dataset B:  & 0.0522 (0.0134)	 & 0.0498 (0.0148)	 & 0.0152 (0.0038) & 	0.0237 (0.0049)\\
Some inputs are observed confounders& & & & \\
Dataset C:  & 0.0591 (0.0145) & 	0.0671 (0.0180)	 & 0.0315 (0.0055) & 	0.0328 (0.0113)\\
Some inputs are observed confounders& & & & \\
Dataset D:  & 	0.0375 (0.0141)	& 0.0539 (0.0075) & 	0.0096 (0.0022) & 	0.029 (0.0058)\\
Some input proxies are& & & & \\
not conditionally independent& & & & \\
 \bottomrule
\end{tabular}}
\end{sc}
\end{center}
\vskip -0.1in
\end{table}

\begin{table}[h]
\caption{DMSE under alternative graph structures}
\label{table:alt-graphs}
\vskip 0.15in
\begin{center}
\tabcolsep=1.5pt
\begin{small}
\begin{sc}
\resizebox{\textwidth}{!}
{\begin{tabular}{lccr}
\toprule
Causal Graph & 	Training Dataset Size & 	ATE error (train+val) & 	ATE error (test)\\
\midrule
Dataset A:  & 	100	 & 0.1966 (0.0389) & 	0.3590 (0.0564)\\

Original causal graph & 1000 & 	0.0575 (0.0103) & 	0.1155 (0.0220)\\
 & 10000 & 	0.0335 (0.0127) & 	0.0303 (0.0140)\\
 & 25000 & 	0.0274 (0.0055)	 & 0.0292 (0.0058)\\
 \midrule
Dataset B: & 	100 & 	0.1280 (0.0184) & 	0.2770 (0.0401)\\
 Some inputs& 1000 & 	0.0320 (0.0099)	 & 0.0951 (0.0264)\\
are observed confounders & 10000 & 	0.0152 (0.0038) & 	0.0237 (0.0049)\\
 & 25000 & 	0.0195 (0.0058) & 	0.0204 (0.0063)\\
 \midrule
Dataset C:  & 100 & 	0.1354 (0.0361) & 	0.1970 (0.0561)\\
Some inputs  & 1000 & 	0.0596 (0.0191) & 	0.1060 (0.0214)\\
 are observed confounders	& 10000 & 	0.0315 (0.0055)	 & 0.0328 (0.0113)\\
 & 25000 & 	0.0140 (0.0039) & 	0.0226 (0.0072)\\
 \midrule
Dataset D:  & 	100 & 	0.1391 (0.0209) & 	0.2300 (0.0544)\\
  Some input proxies& 1000 & 	0.0596 (0.0103)	 & 0.1012 (0.0223)\\
 are not conditionally & 10000 & 	0.0096 (0.0022)	 & 0.0290 (0.0058)\\
 independent& 25000 & 	0.0139 (0.0033)	 & 0.0206 (0.0045)\\
 \bottomrule
\end{tabular}}
\end{sc}
\end{small}
\end{center}
\vskip -0.1in
\end{table}

\subsubsection{Experimental Results}
\label{apdx:additional_cevae_comparison}
\textbf{DMSE under alternative graph structures and comparison with CEVAE}:We generate synthetic data corresponding to Datasets A, B, C and D. Table~\ref{table:alt-graphs} demonstrates that DMSE recovers ATE under modified causal graph structures. We demonstrate that with increasing dataset size, the ATE error on test dataset continues to fall. Thus, with dataset size approaching infinity, we can recover the true ATE as long as the model class contains true distribution and our optimizer can find the minimum. In Table \ref{table:alt-graphs-comp}, DMSE also compares favorably with CEVAE and recovers the ATE in this extended setting. In this extended setting, we have a combination of structured and unstructured input modalities. CEVAE takes a concatenation of these modalities as its input while DMSE has a separate model and inference network for each modality. Hence DMSE can handle diverse modality types gracefully as compared to CEVAE.

\textbf{DMSE under increasing number of modalities and comparison with CEVAE}:
We generate synthetic data corresponding to Dataset E with varying number of modalities (i.e. proxies).  Table~\ref{table:one-modality-comparison} contains the results of experiment with just one modality; the key difference between the two models is the inference procedure. Table~\ref{table:one-modality-comparison} shows that even in this setting, DMSE recovers the ATE more accurately than CEVAE. Next, we compare CEVAE vs. DMSE when the data has many unstructured modalities in Table~\ref{table:increasing-modalities}. We generate synthetic data from K modalities (Dataset E). The CEVAE model treats them as one concatenated vector; DMSE models them as separate vectors. As expected, DMSE handles large numbers of modalities better than CEVAE.

\section{Identifiability of Causal Effects in the Presence of Proxies to Hidden Confounder}
\label{app:theory}

If the structural equations have linear dependencies, we can establish the identifiability of total treatment effect. \citet{kuroki2014measurement} define the total treatment effect of $X$ on $Y$ as `the total sum of the products of the path coefficients on the sequence of arrows along all directed paths from X to Y'. 
We extend the identifiability result of~\citet{kuroki2014measurement} to handle vector-valued confounders $\bar U$, and hence we require one additional view of the confounder for a total of three independent views.
Concretely, we consider the following setup.

$\bar{U}$: hidden confounder

$X$: binary treatment variable

$Y$: univariate outcome variable

$\bar{W}, \bar{Z}, \bar{V}$: proxies for $\bar{U}$

Assume causal graph as 

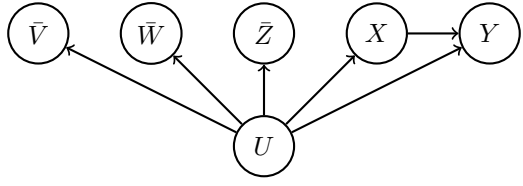
\begin{figure}[!h]
    \centering
    \begin{tikzpicture}[thick, node distance={15mm}, minimum size=0.8cm, main/.style = {draw, circle}] 
    \node[main] (1) {$\bar V$};
    \node[main] (2) [right of=1] {$\bar W$};
    \node[main] (3) [right of=2] {$\bar Z$};
    \node[main] (4) [right of=3] {$X$};
    \node[main] (6) [right of=4] {$Y$};
    \node[main] (5) [below of=3]{$U$}; 
    \draw [->] (5) -- (1);
    \draw [->] (5) -- (2);
    \draw [->] (5) -- (3);
    \draw [->] (5) -- (4);
    \draw [->] (4) -- (6);
    \draw [->] (5) -- (6);
    \end{tikzpicture} 
     \caption{Causal graph for Theorem~\ref{thm:linear_sem}.}
    \label{fig:linear_sem}
\end{figure}

Assume linear structural equations

$\bar{W} = \beta_{\bar{W}\bar{U}} \bar{U} + \bar{\epsilon}_W$

$\bar{Z} = \beta_{\bar{Z}\bar{U}} \bar{U} + \bar{\epsilon}_Z$

$\bar{V} = \beta_{\bar{V}\bar{U}} \bar{U} + \bar{\epsilon}_V$

$X = \beta_{X\bar{U}} \bar{U} + \epsilon_X$

$Y = \beta_{Y\bar{U}} \bar{U} + \beta_{YX} X + \epsilon_Y$

Note that the three views of $\bar U$ are independent given $\bar U$, i.e. $\bar{W} \perp \bar{Z} \perp \bar{V} | \bar{U}$.

\begin{theorem}\label{thm:linear_sem}
  Given the above setup, the causal effect of $X$ on $Y$, i.e. $P(y|\textrm{do}(x))$, is identifiable if the dependency matrices $\beta_{\diamond \bar U}$ for $\diamond \in \{ \bar V, \bar W, \bar Z\}$ have rank $|\bar U|$.
\end{theorem}

\newcommand{\BlackBox}{\rule{1.5ex}{1.5ex}}
\newenvironment{proofsketch}{\par\noindent{\emph{Proof Sketch}\ }}{\hfill\BlackBox\\[2mm]}

\begin{proofsketch}
First, we analyze what happens when we know the $\beta$ matrices a priori, i.e. treatment effect estimation with external information.
Then, we analyze how to leverage the three independent views of the confounder to recover enough information about the matrices for treatment effect estimation, i.e. treatment effect estimation without external information. In this second step, we use the rank constraint so that each proxy has enough information about the confounder. 

\textbf{Treatment effect estimation with external information}:

If any one $\beta$ matrix corresponding to a proxy is known (i.e. we know the causal mechanism $\bar{U} \rightarrow \bar{W}$, $\bar{U} \rightarrow \bar{Z}$ or $\bar{U} \rightarrow \bar{V}$), then the treatment effect is identifiable given that the $\beta$ matrix is invertible. In principle, we need any one proxy and the $\beta$ matrix corresponding to it for computing the treatment effect. The following derivation is done using the matrix $\beta_{\bar{W}\bar{U}}$.

We define $\sigma_{AB} = \textrm{Covariance}(A, B)$ where $A$ and $B$ are univariate. We define $\Sigma_{\bar{A}\bar{B}}$ to be the covariance matrix between sets of variables in the vectors $\bar{A}$ and $\bar{B}$. 

For the above causal graph Figure~\ref{fig:linear_sem} with linear dependencies, the total treatment effect $\tau_{YX}$ of variable $X$ on output $Y$ can be derived using backdoor adjustment~\cite{kuroki2014measurement} as

$\tau_{YX} = \frac{\sigma_{XY} - \Sigma_{X\bar{U}} \Sigma^{-1}_{\bar{U} \bar{U}} \Sigma_{Y\bar{U}}^T}{\sigma_{XX} - \Sigma_{X\bar{U}} \Sigma^{-1}_{\bar{U}\bar{U}} \Sigma_{X\bar{U}}^T}. $


Using formula of covariance and owing to the linear nature of structural equations, we know that

$\Sigma_{X\bar{W}} = \Sigma_{X \bar{U}} \beta_{\bar{W}\bar{U}}^T ,  \Sigma_{Y\bar{W}} =  \Sigma_{Y\bar{U}} \beta_{\bar{W}\bar{U}}^T.$

Thus, if the matrix $\beta_{\bar{W}\bar{U}}$ is a square and invertible, we get  

$\tau_{YX} = \frac{\sigma_{XY} - \Sigma_{X\bar{W}} (\beta_{\bar{W}\bar{U}}\Sigma_{\bar{U} \bar{U}}\beta_{\bar{W}\bar{U}}^T)^{-1}\Sigma_{Y\bar{W}}^T}{\sigma_{XX} - \Sigma_{X\bar{W}} (\beta_{\bar{W}\bar{U}}\Sigma_{\bar{U} \bar{U}}\beta_{\bar{W}\bar{U}}^T)^{-1}\Sigma_{X\bar{W}}^T}. $

\textbf{Note}: \label{Note} If the matrix $\beta_{\bar{W}\bar{U}}$ is not square, it needs to have rank $|\bar{U}|$ at the least. In case $|\bar{W}| > |\bar{U}|$, but rank($\beta_{\bar{W}\bar{U}}$) = $|\bar{U}|$, we can still work with a modified system of equations such that $\bar{W}$ is replaced with $\bar{W'}$ where $|\bar{W'}| = |\bar{U}| = \textrm{rank}(\beta_{\bar{W'}\bar{U}})$. To obtain this modified system, we apply elementary row transformations on the equation $\bar{W} = \beta_{\bar{W}\bar{U}} \bar{U} + \bar{\epsilon}_W$ such that $\beta_{\bar{W}\bar{U}}$ is an upper triangular matrix. After that, we drop the bottom $|\bar{W}|-|\bar{U}|$ rows from matrices on both sides of equation to obtain $\bar{W'}$ and  $\beta_{\bar{W'}\bar{U}}$.

Consider the case where rank($\beta_{\bar{W}\bar{U}}$) $<$ $|\bar{U}|$ and  rank($\beta_{\bar{Z}\bar{U}}$) $<$ $|\bar{U}|$. In this case, \textbf{if} we can concatenate the two proxies $\bar{W}$ and $\bar{Z}$ to a new proxy $\bar{V} = \bar{W}:\bar{Z}$ such that rank($\beta_{\bar{V}\bar{U}}$) = $|\bar{U}|$, we can estimate treatment effects using $\beta_{\bar{V}\bar{U}}$.

\textbf{Treatment effect estimation without external information}:

In this case, we need three independent views of hidden confounder U. As we are working with multivariate confounders, the univariate treatment variable X cannot serve as the third view of U anymore.

Now using the properties of covariance matrix and the linearity of structural equations, we can write the following 

$\Sigma_{\bar{V}\bar{W}} = \beta_{\bar{V}\bar{U}} \Sigma_{\bar{U}\bar{U}}\beta_{\bar{W}\bar{U}}^T$, $\Sigma_{\bar{W}\bar{Z}} = \beta_{\bar{W}\bar{U}} \Sigma_{\bar{U}\bar{U}}\beta_{\bar{Z}\bar{U}}^T $, $\Sigma_{\bar{V}\bar{Z}} = \beta_{\bar{V}\bar{U}} \Sigma_{\bar{U}\bar{U}}\beta_{\bar{Z}\bar{U}}^T $.

Assuming that the matrices $\Sigma_{\bar{U}\bar{U}}, \beta_{\bar{V}\bar{U}}, \beta_{\bar{Z}\bar{U}}$ are square and invertible, we get $  \Sigma_{\bar{W}\bar{Z}} \Sigma_{\bar{V}\bar{Z}}^{-1}  \Sigma_{\bar{V}\bar{W}} = \beta_{\bar{W}\bar{U}} \Sigma_{\bar{U}\bar{U}} \beta_{\bar{W}\bar{U}}^T$

Hence, if $\beta_{\bar{W}\bar{U}}$ is invertible additionally, then the treatment effect can be identified using the following equation

$\tau_{YX} = \frac{\sigma_{XY} - \Sigma_{X\bar{W}} (\Sigma_{\bar{W}\bar{Z}} \Sigma_{\bar{V}\bar{Z}}^{-1}  \Sigma_{\bar{V}\bar{W}})^{-1}\Sigma_{Y\bar{W}}^T}{\sigma_{XX} - \Sigma_{X\bar{W}} (\Sigma_{\bar{W}\bar{Z}} \Sigma_{\bar{V}\bar{Z}}^{-1}  \Sigma_{\bar{V}\bar{W}})^{-1}\Sigma_{X\bar{W}}^T}$

\end{proofsketch}

\textbf{What happens when the dimensionality of proxies is greater than dimensionality of true confounder?}

We are interested in high-dimensional, unstructured proxies. In this case, our matrices $\beta_{\bar{V}\bar{U}}, \beta_{\bar{Z}\bar{U}}, \beta_{\bar{W}\bar{U}}$ will need to have rank $|\bar{U}|$. In essence, we still need three views for our vector $\bar{U}$. When we have access to external information in the form of $\beta$ matrices, verifying the rank is possible. We can then apply appropriate row transformations  on the corresponding structural equations (as mentioned in~\ref{Note}) to obtain modified proxies with same length as the confounder $\bar{U}$. 

In absence of external information, rank of $\beta$ matrices is not verifiable. Assuming that we have $\beta_{\bar{V}\bar{U}}, \beta_{\bar{Z}\bar{U}}, \beta_{\bar{W}\bar{U}}$ with rank $|\bar{U}|$, we apply a dimensionality reduction procedure to map the proxies $\bar{V}, \bar{Z}, \bar{W}$ to modified proxies $\bar{V'}, \bar{Z'}, \bar{W'}$ such that $|\bar{V'}|= |\bar{Z'}|= |\bar{W'}| = |\bar{U}|$ and $\beta_{\bar{V'}\bar{U}}, \beta_{\bar{Z'}\bar{U}}, \beta_{\bar{W'}\bar{U}}$ have rank $|\bar{U}|$. In practice, this dimensionality reduction can be a technique like PCA or the result of applying pre-trained neural network layer to obtain an embedding of unstructured data. 
\section{Additional Details on Mathematical Proofs}

\subsection{Evidence Lower Bound for Deep Structural Equations}
We discuss the evidence lower bound (ELBO), presented in Section~\ref{subsection:approx_inference}.

The left-hand-side of the ELBO can be written as
\begin{align}
    \displaystyle\sum_{i=1}^n \mathbb{E}_q \left[ \displaystyle\sum_{j=1}^m \log p(x_i^j | z) + \log p(y_i | z) + \log p(y_i, t_i, z) - \log q(z | x_i, t_i, y_i) \right].
\end{align}
Due to the causal graph structure we can factorize the following distribution as
\begin{align}
    p(y_i, t_i, z) = \log p(y_i | t_i, z) + \log p(t_i | z) + \log p(z).  
\end{align}
Thus, we can rewrite (10) as
\begin{align}
    \displaystyle\sum_{i=1}^n \mathbb{E}_q \left[ \displaystyle\sum_{j=1}^m \log p(x_i^j | z) + \log p(y_i | z) + \log p(y_i | t_i, z) + \log p(t_i | z) + \log p(z) - \log q(z | x_i, t_i, y_i) \right]
\end{align}
Note that the first four terms in the expectation over the posterior distribution $q$ make up a reconstruction loss of the original data $\boldsymbol{x}, y, t$ in terms of the latent variable $z$, while the last two terms form the negative KL divergence $-D_{KL} \left( q(z | x_i, t_i, y_i) \lVert p(z) \right)$ between the posterior $q$ and the prior $p(z)$. Thus, we can write our ELBO as
\begin{align}
    \displaystyle\sum_{i=1}^n \mathcal{L}_{\textrm{reconstruction}}(x_i, y_i, t_i, z) - D_{KL} \left( q(z | x_i, t_i, y_i) \lVert p(z) \right),
\end{align}
which is exactly our variational objective. We can sum or take the average over the number of datapoints $n$ to form our empirical objective.

\subsection{Derivation of posterior}
Here we show why the true posterior factorizes as $$p(z\mid x,t,y) \propto (p(z\mid t,y) \prod_{j=1}^m p(z\mid x_j)) / \prod_{j=1}^{m-1}p(z).$$ 
\begin{align*}
    p(z\mid x,t,y) 
    &= p(x,t,y,z)/p(x,t,y)
    \\&= \left(\prod_{j=1}^m p(x_j\mid z)\right) p(t, y\mid z) p(z) /p(x, t, y) \tag{by cond. indep. of causal graph}
    \\&= \left( \prod_{j=1}^m p(z\mid x_j) p(x_j) / p(z) \right) \cdot \left(p(z\mid t, y) p(t, y) p(z) / p(x, t, y)\right)
    \\&\propto \left( \prod_{j=1}^m p(z\mid x_j) / p(z) \right) \cdot \left(p(z\mid t, y) p(z) \right). \tag{removing terms independent of $z$}
\end{align*}

\end{document}